\documentclass{article} % For LaTeX2e
\usepackage{iclr2027_arxiv,times, macros}

\usepackage{amsmath,amsfonts,bm}

\def\eqref#1{equation~\ref{#1}}
\def\1{\bm{1}}

\def\vd{{\bm{d}}}
\def\ve{{\bm{e}}}

\def\vk{{\bm{k}}}

\def\vo{{\bm{o}}}

\def\vq{{\bm{q}}}
\def\vr{{\bm{r}}}

\def\vu{{\bm{u}}}
\def\vv{{\bm{v}}}
\def\vw{{\bm{w}}}
\def\vx{{\bm{x}}}
\def\vy{{\bm{y}}}

\def\mA{{\bm{A}}}

\def\mE{{\bm{E}}}

\def\mG{{\bm{G}}}

\def\mO{{\bm{O}}}

\def\mR{{\bm{R}}}
\def\mS{{\bm{S}}}

\def\mV{{\bm{V}}}
\def\mW{{\bm{W}}}
\def\mX{{\bm{X}}}

\def\mZ{{\bm{Z}}}

\DeclareMathAlphabet{\mathsfit}{\encodingdefault}{\sfdefault}{m}{sl}
\SetMathAlphabet{\mathsfit}{bold}{\encodingdefault}{\sfdefault}{bx}{n}

\usepackage[utf8]{inputenc} % allow utf-8 input
\usepackage[T1]{fontenc}    % use 8-bit T1 fonts
\usepackage{hyperref}       % hyperlinks
\usepackage{cleveref}       % smart cross-references (\cref)
\usepackage{url}            % simple URL typesetting
\usepackage{booktabs}       % professional-quality tables
\usepackage{amsfonts}       % blackboard math symbols
\usepackage{nicefrac}       % compact symbols for 1/2, etc.
\usepackage{microtype}      % microtypography
\usepackage{xcolor}         % colors

\definecolor{mayablue}{rgb}{0.21,0.49,0.74}
\NewDocumentCommand{\hl}{ mO{} }{\textcolor{mayablue}{\textsuperscript{\textit{Hanlin}}\textsf{\textbf{\small[#1]}}}}

\usepackage{etoc}

\usepackage{amsmath,amssymb}
\usepackage{algorithm}
\usepackage{algpseudocode}
\usepackage{graphicx}
\usepackage{wrapfig}
\usepackage{booktabs}
\usepackage{makecell}

\usepackage[most]{tcolorbox}

\definecolor{takeawayborder}{RGB}{196,114,34}
\definecolor{takeawaybg}{RGB}{255,243,224}

\newtcolorbox{takeawaybox}{
  colback=takeawaybg,
  colframe=takeawayborder,
  boxrule=0.9pt,
  arc=2pt,
  left=6pt,
  right=6pt,
  top=5pt,
  bottom=5pt
}

\title{A Unifying View of Attention Sinks: From\\ Mechanisms to Architectural Interventions}

\author{%
  Lukas Fesser\thanks{Equal contribution and corresponding author.}
  \quad
  Mozes Jacobs\footnotemark[1]
  \quad
  Thomas Fel\footnotemark[1]
  \quad
  Andy Keller
  \quad
  Sham Kakade
  \\
  Kempner Institute \\
  Harvard University \\
  Cambridge, MA \\
  \small{\texttt{\{lukas\_fesser, mozesjacobs, tfel\}@g.harvard.edu}}
}

\begin{document}

\etocdepthtag.toc{main_iclr_pathway}

\maketitle

\begin{abstract}
When attention concentrates on a single token, a \emph{sink}, what is the model actually computing?
Attention sinks are ubiquitous in softmax transformers, yet this shared visual signature can hide fundamentally different algorithms.
We show that visually similar sink patterns can reflect two distinct mechanisms: \i{i} adaptive ~\nop\footnote{From the \nop~instruction in assembly. On many architectures, \nop~is encoded as all zeros (unmapped memory returns zeroes by default), so a jump into the void does nothing gracefully.}, where a head suppresses its update by routing to a null token, and \i{ii} broadcast, where a sink aggregates and redistributes global information.
Each mechanism leaves distinct traces (\nop~sinks exhibit negligible value norms; broadcast sinks induce low-rank outputs), which we formalize on synthetic tasks and use to derive practical diagnostics. Applied to pretrained vision transformers, these diagnostics reveal that both mechanisms exist at scale: sinks transition from \texttt{CLS} in early layers to patches in deeper layers and concentrate in specialized heads.
Causal interventions further connect these signatures to near-null suppression and shared residual contributions. We then use architectural interventions to show how these computations can be reorganized: gating eliminates detected \nop-like sinks but increases broadcast-like sinks, registers relocate rather than remove sink computation, and our position-free global pathway provides an explicit route for shared communication that reduces the broadcast-like sinks induced by gating. On dense probes, combining gating with the global pathway gives the strongest results among the tested variants, despite retaining some broadcast-like sinks.
Overall, we find that the same attention pattern can reflect two very different computations, and that effective intervention depends not only on identifying the computation, but also on providing architectural alternatives through which the model can reorganize it.
\end{abstract}

\section{Introduction}

\looseness=-1
Attention sinks, tokens that attract disproportionate attention across many queries, are pervasive in softmax transformers~\citep{vaswani2017attention}, including vision transformers (ViTs)~\citep{dosovitskiy2020image,zhai2022scaling,dehghani2023scaling,tschannen2025siglip2} and large language models (LLMs)~\citep{brown2020language,touvron2023llama,ouyang2022training}. They appear as near-vertical ``stripes'' in attention maps~\citep{barbero2025llmsattendtoken,queipo2025attention,cancedda2024spectral,lu2025artifacts}, often at the \texttt{[CLS]} token or specific patches~\citep{darcet2023vision,jiang2025vision,wu2025emergence}. Yet the same signature can reflect different computations: what is the model computing, and how should the architecture accommodate it?

\begin{figure}[t]
\centering
% \vspace{0mm}
\includegraphics[width=\linewidth]{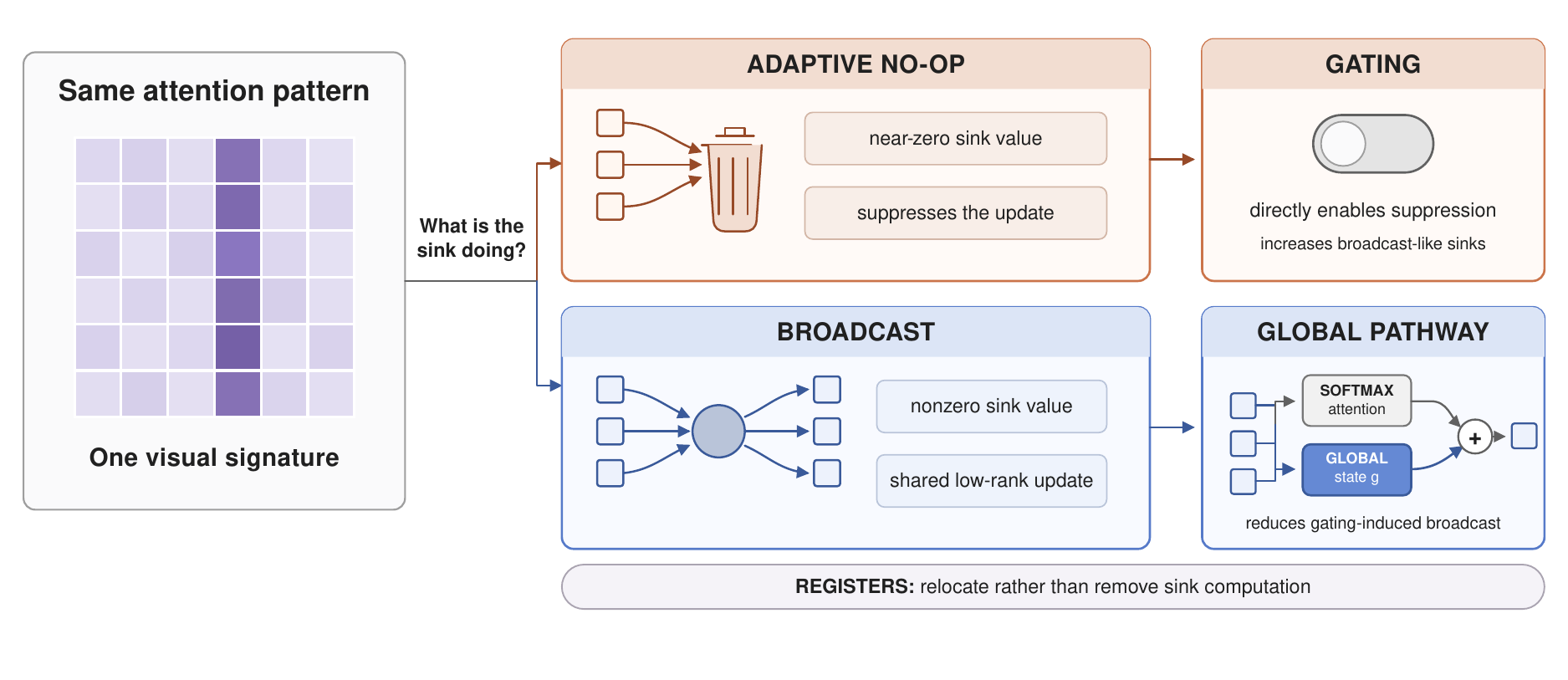}
\vspace{-4mm}
\caption{
\textbf{Same visual signature, different mechanisms.} Attention sinks can implement adaptive \nop{}, suppressing an update through a near-zero value, or broadcast, producing a shared low-rank update. Gating, registers, and a position-free global pathway provide alternative ways to reorganize these computations.
}
\label{fig:teaser}
\vspace{-5mm}
\end{figure}

Sink-like patterns have been linked to spectral concentration, rank collapse, anisotropy, massive activations, norm imbalance, and softmax saturation~\citep{qiu2026unified,zhang2025attention,cancedda2024spectral,queipo2025attention,lu2025artifacts,sun2024massive,dettmers2022gpt3,bondarenko2023quantizable,rybakov2024methods,kang2025see}, with consequences for stability and representation quality~\citep{gallego2025hidden}. However, sinks can also serve as useful workspaces that aggregate and redistribute global information~\citep{darcet2023vision}. Mitigation therefore requires distinguishing harmful effects from useful computation.

Attention patterns alone cannot make this distinction~\citep{akula2022attention,bibal2022attention,wiegreffe2019attention}. Because softmax must allocate probability mass somewhere, a sink may carry information or provide a safe destination when a head should do nothing~\citep{guo2024active}. We study these alternatives as two canonical mechanisms: \textbf{adaptive \nop{}}, where routing to a near-zero value suppresses the residual update, and \textbf{broadcast}, where routing to a non-negligible value distributes a shared component across tokens. These mechanisms predict distinct value-norm and output-rank signatures.

We test these interpretations through causal interventions in pretrained ViTs, manipulating sink values and attention routing to distinguish near-null suppression from nonzero shared contributions. These experiments connect the observed sink signatures to their effects on downstream representations.

We then introduce a position-free global pathway to mitigate broadcast sinks, and study its interaction with gating~\citep{qiu2025gated}, alongside register-based alternatives~\citep{darcet2023vision}. These experiments suggest that preserving representation quality during sink mitigation depends on how architectural alternatives for suppression and shared computation are combined. Our contributions are:

\begin{itemize}
\item We isolate adaptive \nop{} and broadcast in controlled tasks and derive diagnostics that distinguish their computations beyond attention patterns alone.
\item We identify both regimes in pretrained ViTs, test their causal roles, and show that registers retain both.
\item We introduce a penalized global pathway and causally validate broadcast relocation in controlled tasks.
\item We reveal interactions among architectural interventions, finding that combining gating with the penalized pathway reduces broadcast-like sinks relative to gating alone while achieving higher dense-probe scores.
\end{itemize}

\section{Preliminaries and Related Work}

\subsection{Softmax Attention and Attention Sinks}

We consider a Vision Transformer (ViT)~\citep{dosovitskiy2020image} operating on token representations $\vx_1,\ldots,\vx_n\in\mathbb{R}^d$. Unlike autoregressive language models~\citep{qiu2025gated,barbero2025llmsattendtoken,touvron2023llama}, ViTs use bidirectional attention. For a single head,
$\vq_i=\mW_Q^\top\vx_i$, $\vk_i=\mW_K^\top\vx_i$, and $\vv_i=\mW_V^\top\vx_i$, with $\mW_Q,\mW_K,\mW_V\in\mathbb{R}^{d\times d_h}$, and
\begin{equation}
\label{eq:attention}
\vo_i^{\mathrm{head}}
=
\sum_{j=1}^{n}\A_{ij}\vv_j,
\qquad
\A_{ij}
=
\frac{\exp(\vq_i^\top\vk_j/\sqrt{d_h})}
{\sum_{l=1}^{n}\exp(\vq_i^\top\vk_l/\sqrt{d_h})}.
\end{equation}
Writing the corresponding row-stacked matrices as $\mA$ and $\mV$, the head output is $\mO^{\mathrm{head}}=\mA\mV$. Because $\mA$ is row-stochastic, each query distributes unit attention mass across the sequence. We use the standard residual convention $\vx_i^{\ell+1}=\vx_i^\ell+\vo_i^\ell$, omitting other sublayers for clarity.

We define an attention sink, consistent with prior work~\citep{ruscio2025you,gu2025when}, as a position that attracts disproportionate attention across many queries without assuming an underlying mechanism.
\begin{definition}[\textbf{$\epsilon$-Attention Sink}]
\label{def:epsilon_sink}
Let $\epsilon\in[0,1)$. A position $s\in\{1,\ldots,n\}$ is an \emph{$\epsilon$-attention sink} with respect to query indices $\mathcal{I}\subseteq\{1,\ldots,n\}$ if $\A_{is}\geq1-\epsilon$ for all $i\in\mathcal{I}$.
\end{definition}
Thus $\epsilon=0$ gives a perfect sink. We additionally measure continuous sink strength as
$\mathrm{sink}(s;\mathcal{I})=|\mathcal{I}|^{-1}\sum_{i\in\mathcal{I}}\A_{is}$,
and use $\max_s\mathrm{sink}(s;\mathcal{I})$ for the strongest sink in a head. This definition follows prior work~\citep{gu2025when,barbero2025llmsattendtoken} and remains agnostic to whether concentrated attention serves suppression or information transfer.

\subsection{Related Work}

\paragraph{Attention sinks in LLMs.}
Attention sinks in autoregressive language models often occur at early positions or special tokens even when those positions are not semantically central~\citep{barbero2025llmsattendtoken,gu2025when}. Prior work links them to causal masking, long-context behavior, and architectural biases~\citep{barbero2025llmsattendtoken,qiu2025gated,wu2024role}, as well as spectral concentration, residual imbalance, outlier activations, and conditionally inactive heads~\citep{queipo2025attention,cancedda2024spectral,sun2024massive,dettmers2022gpt3,bondarenko2023quantizable,sandoval2025using}. These observations have motivated gated and modified attention mechanisms that reduce concentration or allow heads to suppress their outputs directly~\citep{bondarenko2023quantizable,qiu2025gated}.

\paragraph{Attention sinks in ViTs.}
In ViTs, sinks appear at \texttt{[CLS]}, particular patches, and learned workspace tokens such as registers~\citep{lu2025artifacts,darcet2023vision,jiang2025vision,kang2025see}. They have been interpreted both as artifacts associated with outliers or softmax saturation~\citep{lu2025artifacts,bondarenko2023quantizable} and as useful mechanisms for aggregating and redistributing global information~\citep{darcet2023vision}. Registers provide dedicated positions that absorb attention mass and can improve representation quality~\citep{darcet2023vision,jiang2025vision}, but remain sequence positions and can themselves become sinks. Related work on rank collapse, oversmoothing, position bias, and ViT dynamics further connects sink behavior to token mixing and representational diversity~\citep{jacobs2025block,wang2022anti}. What remains missing is a practical way to distinguish the computations underlying visually similar sinks and architectural alternatives for realizing suppression and shared communication; we address both.

\section{Two Hypotheses for Attention Sinks}

The signature of an attention sink, a vertical stripe in the attention map, is computationally ambiguous. Does the model attend to a token because it carries useful information, or because it provides a safe destination when no update is needed? Our goal is to derive stylized mechanistic intuition and operational diagnostics that distinguish these possibilities in pretrained ViTs and motivate architectural alternatives for the computations they implement. We formalize two regimes below. Throughout, $\vv_s$ denotes the single-head value in head-space analyses and the output-projected value when discussing residual-stream updates.

\begin{hypothesis}[\textbf{Adaptive \nop}]
\label{hyp:noop}
An attention head implements an adaptive \nop{} mechanism~\citep{sandoval2025using,barbero2025llmsattendtoken} when it selectively routes queries to a sink position $s$ whose contribution is negligible. For tokens $i$ attending predominantly to $s$, $\vo_i \approx \bm{0}$, so the head approximately preserves the residual stream.
\end{hypothesis}

\begin{hypothesis}[\textbf{Broadcast}]
\label{hyp:broadcast}
An attention head implements a broadcast mechanism~\citep{darcet2023vision} when attention from many positions concentrates on a sink position $s$ whose value vector $\vv_s$ induces a shared residual contribution. For tokens attending predominantly to $s$, the residual update satisfies $\vx_i^{\ell+1}-\vx_i^\ell \approx \A_{is}\vv_s$. The recipient representations therefore acquire a shared component aligned with $\vv_s$.
\end{hypothesis}

Both mechanisms produce concentrated attention, but they assign different computational roles to the sink: a null destination for conditional suppression or a source of a shared update. Distinguishing them requires examining the values and resulting updates, rather than attention weights alone.

\subsection{Hypothesis 1: Sink as Adaptive \nop}
\label{sec:noop}

Attention contributes additively to the residual stream through $\vx_i^{\ell+1}=\vx_i^\ell+\vo_i^\ell$. When a head has no useful information to contribute, softmax still requires it to allocate probability mass~\citep{barbero2025llmsattendtoken}. A sink with a near-zero value provides a way to satisfy this constraint while producing a negligible update. Crucially, this behavior is adaptive: the same head can perform a meaningful operation for some queries and suppress its contribution for others~\citep{sandoval2025using}.

\begin{wrapfigure}{r}{0.5\linewidth}
\vspace{-6mm}
\centering
\includegraphics[width=0.95\linewidth]{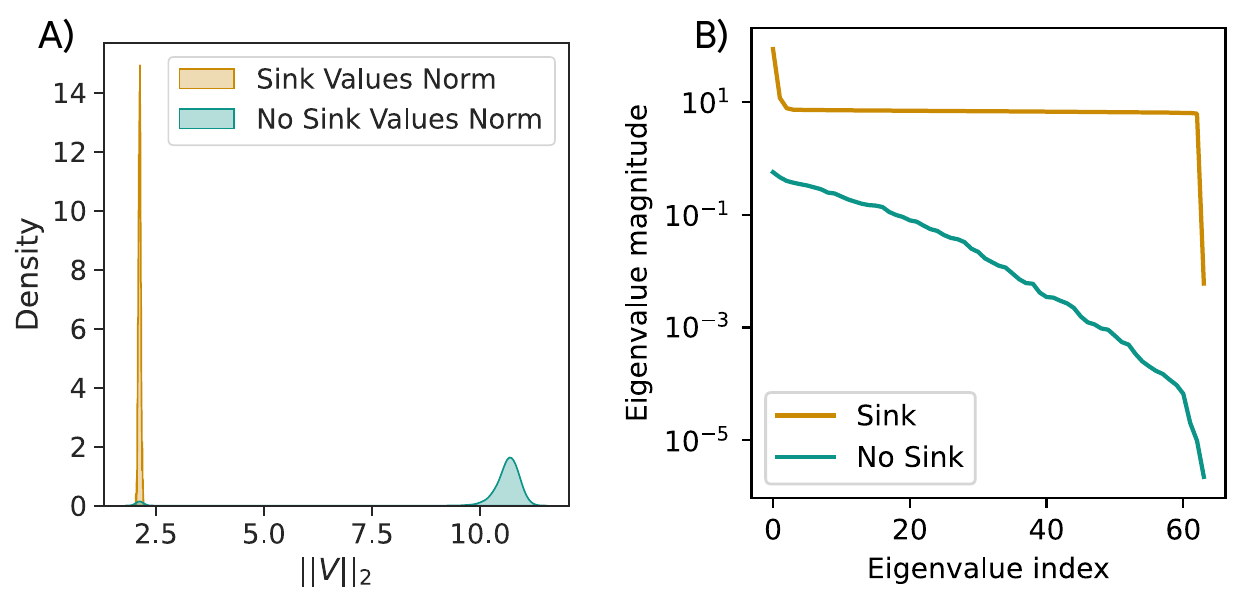}
\caption{
\textbf{\nop{} sink signatures.}
\textbf{(A)} Sink solutions learn near-zero sink value norms, producing negligible updates.
\textbf{(B)} Sink models exhibit a dominant singular value in $\mW_Q\mW_K^\top$, consistent with a learned gating direction.
\vspace{-2mm}
}
\vspace{-4mm}
\label{fig:no_op_signatures}
\end{wrapfigure}

\paragraph{\nop{} Signatures.}
A negligible output can in principle arise through cancellation among several values or by attending to a null value. Once attention concentrates on a single position, however, cancellation by the remaining tokens is limited. The following result makes precise why a strongly concentrated \nop{} sink must have a small value.

\begin{lemma}[\textbf{Uniqueness of the Exact No-Op Solution}]
\label{lem:noop_unique}
Consider a single attention head with output
$\vo_i = \sum_j \A_{ij}\vv_j$, where values are measured in the
same space as the output. Let $s$ be an $\epsilon$-sink for a
nonempty query set $I$, with $0 \leq \epsilon < 1$, and suppose
$\|\vo_i\| \leq \delta$ for all $i \in I$.
Writing $V_{\max} = \max_j \|\vv_j\|$, we have
\begin{equation}
    \|\vv_s\| \leq
    \frac{\delta + \epsilon V_{\max}}{1-\epsilon}.
\end{equation}
If $V_{\max}$ remains uniformly bounded, then
$\epsilon,\delta \to 0$ imply $\vv_s \to \bm{0}$.
For a perfect sink ($\epsilon=0$), the no-op constraint is
equivalent to $\|\vv_s\| \leq \delta$; hence exact no-op
($\delta=0$) has the unique sink value $\vv_s=\bm{0}$.
\end{lemma}

The null value explains how a head suppresses its update; conditional routing explains when it does so. For a query requiring a \nop{}, the sink logit must exceed competing logits sufficiently for $\A_{is}$ to approach one. In a linear query--key interaction, this can be achieved through a dominant singular direction in $\mW_Q\mW_K^\top$ aligned with a triggering feature, or through large sink activations aligned with that direction. Lemma~\ref{lem:adaptivity} in Appendix~\ref{app:adaptivity_proof} gives constructive examples of both regimes. This connects adaptive \nop{} to spectral spikes and to the massive activations observed in outlier tokens~\citep{darcet2023vision}, although neither signature alone establishes that a head performs a \nop{}.

\paragraph{Toy Model of Adaptive \nop.}
To isolate conditional suppression, we train a single-layer attention head on a conditional damping task. Inputs are random tokens $\vx_i \sim \mathcal{N}(\bm{0},\bm{I}_d)$, and a fixed gate vector $\vgate$ determines the target head output. If $\langle\vx_i,\vgate\rangle>0$, the target is $\vy_i=\gamma\vx_i$, with $\gamma=0$ in the default experiment; otherwise, $\vy_i=\vx_i$. The head must therefore produce a null output for flagged tokens while reconstructing the input for the remaining tokens. A permanently inactive head cannot solve the task. Full configurations and the related causal-validation setup appear in Appendix \ref{app:controlled-task-setups}.

\paragraph{Results and Signatures.}
We observe two complementary signatures.
\textbf{(i) A learned null value.}
Across the tested runs, optimization converges to a sink-based solution rather than a dispersed cancellation pattern. The model allocates a token $s$ whose value norm approaches zero (\cref{fig:no_op_signatures}A), providing a shared null destination for flagged queries.
\textbf{(ii) A spectral signature of conditional routing.}
Models with sinks exhibit a dominant singular value in $\mW_Q\mW_K^\top$ (\cref{fig:no_op_signatures}B), consistent with the gating construction in Lemma~\ref{lem:adaptivity}. The learned direction amplifies the feature that determines when queries should attend to the null token.

Together, these observations separate the two ingredients of adaptive \nop{}: a value that permits a negligible update and a routing mechanism that selects it conditionally. Appendix \ref{app:supplementary-controlled} examines how sink formation changes with the frequency and strength of the \nop{} condition, including partially damped targets.

\begin{figure*}[t]
% \vspace{0mm}
\centering
\includegraphics[width=\linewidth]{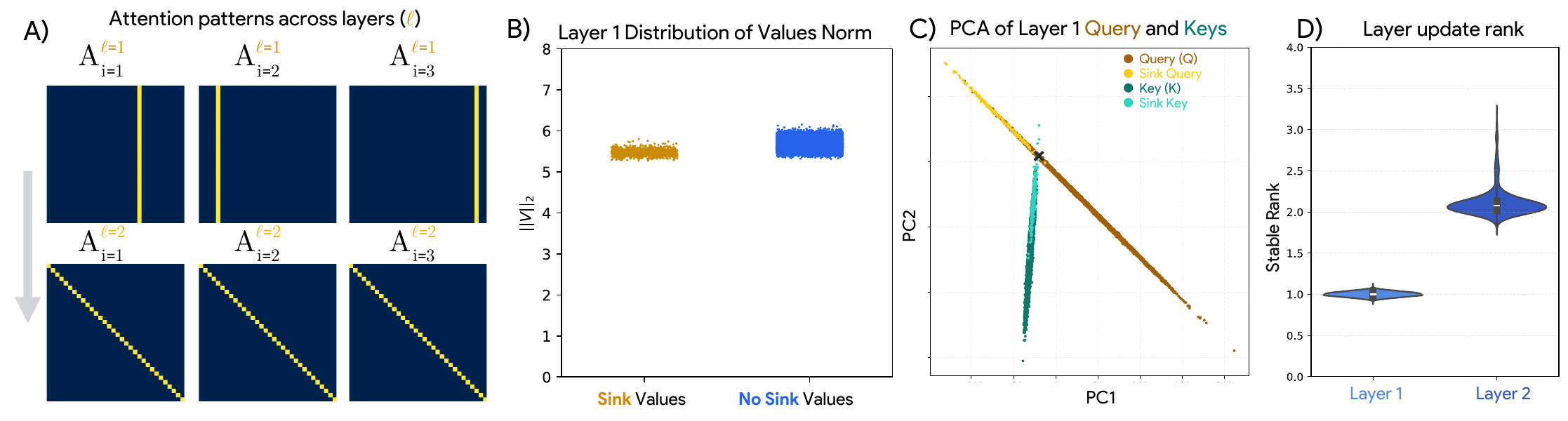}
\caption{
\textbf{Broadcast sink signatures.}
A two-layer model trained on global broadcast learns a modular solution:
\textbf{(A)} Layer 1 forms a broadcast hub while Layer 2 remains identity-like;
\textbf{(B)} sink values retain content-scale norms;
\textbf{(C)} query-key geometry selects the source token; and
\textbf{(D)} the broadcast update is approximately rank-1.
}
\vspace{-3mm}
\label{fig:broadcast_details}
\end{figure*}

\subsection{Hypothesis 2: Sink for Broadcasting}
\label{sec:broadcast}

Under \cref{hyp:broadcast}, a sink produces a shared non-negligible update across recipients. In the controlled setting below, we study the stronger case in which that shared update carries input-specific information, so the sink acts as a communication hub.

\paragraph{Broadcast Signatures.}
This mechanism predicts both a non-negligible sink value and a shared structure in the attention output. When one source dominates attention, recipient updates differ primarily in the scalar weight assigned to that source.

\begin{lemma}[\textbf{Rank-1 Update}]
\label{cor:rank1}
Suppose a single sink position $s$ dominates attention for all queries represented in $\mO=\mA\mV$, with $\A_{is}\approx 1$, uniformly bounded values, and a non-negligible sink value $\vv_s$. Then the output matrix exhibits approximate rank-1 structure:
$\mO\approx\mA_s\vv_s^\top$,
where $\mA_s$ is the sink column of $\mA$.
\end{lemma}

The low-rank structure concerns the attention-induced update, not necessarily the token representations after the residual addition. In particular, a perfectly shared update adds the same vector to every token and preserves centered variance exactly. Broadcast therefore need not increase token alignment or collapse representations toward their mean. Lemma~\ref{cor:similarity} in Appendix~\ref{app:hyp} formalizes this distinction for approximate sinks.

\paragraph{Toy Model of Broadcasting.}
We test whether an objective requiring shared information induces this mechanism. A two-layer Transformer receives random tokens $\vx_i\sim\mathcal{N}(\bm{0},\bm{I}_d)$ and must add a rotated version of a selected source token to every position:
\begin{equation}
\vy_i=\vx_i+\gamma\mR\vx_{j^\star},
\end{equation}
where $\mR\in\R^{d\times d}$ is a fixed orthogonal rotation and $\gamma$ controls broadcast strength. Solving the task requires identifying the source $j^\star$ and distributing its payload while retaining each recipient's input. Source selection and training details appear in Appendix \ref{app:controlled-task-setups}.

\paragraph{Results and Signatures.}
The learned solution exhibits three principal properties.
\textbf{(i) Specialized broadcast routing.}
Although a two-layer model could distribute communication across depth, the illustrated solution concentrates it in one layer. Layer 1 attends predominantly to the source $j^\star$, while Layer 2 adopts an identity-like attention pattern (\cref{fig:broadcast_details}A).
\textbf{(ii) Content-scale sink values.}
Unlike the null values learned on the \nop{} task, broadcast sink values have norms comparable to non-sink content values, with both distributions centered around $\|\vv\|\approx 5.5$ (\cref{fig:broadcast_details}B). Concentrated routing therefore delivers a substantial update rather than suppressing one.
\textbf{(iii) Low-rank updates.}
The attention update in the broadcast layer has stable rank~\citep{neyshabur2015norm,arora2018stronger} approximately one, whereas the second layer produces higher-rank updates (\cref{fig:broadcast_details}D). This agrees with Lemma~\ref{cor:rank1}: recipients receive a shared vector through the broadcast layer.

The query--key geometry provides a complementary view of source selection. Recipient queries align along a common direction, and the selected source key aligns strongly with that direction (\cref{fig:broadcast_details}C). Appendix \ref{app:supplementary-controlled} presents this geometric analysis in detail. These observations characterize the learned solution; the causal interventions in Appendix \ref{app:broadcast-causal} further test whether the selected value carries the broadcast payload.

The two tasks thus motivate a joint diagnostic: negligible sink values indicate \nop-like suppression, while content-scale values together with approximately rank-one updates indicate a broadcast-like shared update. We apply these diagnostics to pretrained ViTs in the next section.

\begin{figure*}[t!]
\centering
% \vspace{0mm}
\includegraphics[width=\linewidth]{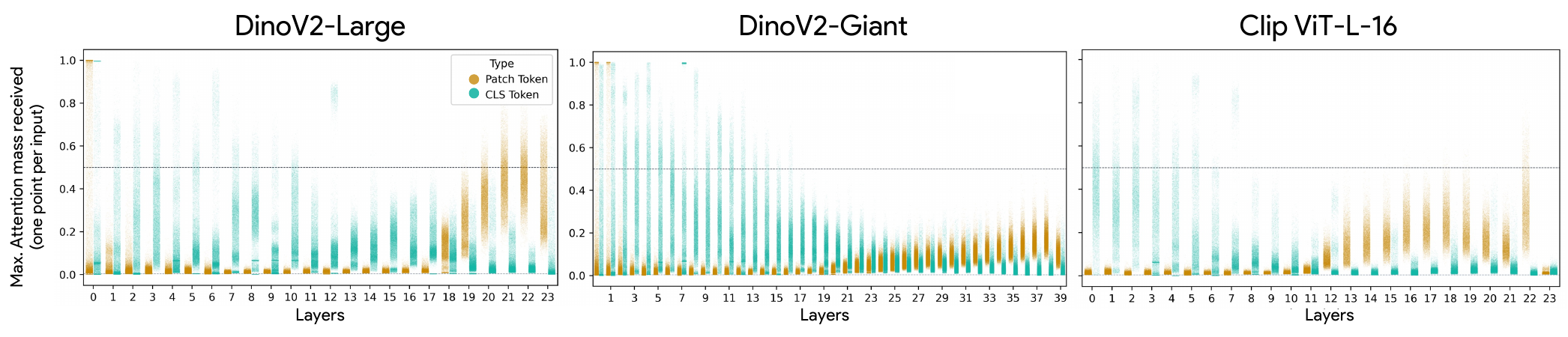}
\caption{
\textbf{Sink token transition across layers.}
DINOv2 (Large and Giant) and OpenCLIP-L-16 exhibit a handoff pattern:
\texttt{[CLS]} serves as the sink in early layers but yields to patch tokens in later layers.
}
\vspace{-4mm}
\label{fig:sink_distribution}
\end{figure*}

\section{Empirical Evidence in ViTs and Phenomenology}
\label{sec:evidence}

Having established the theoretical signatures of the adaptive \nop{} and broadcast mechanisms, we now examine their prevalence in large-scale Vision Transformers.

\paragraph{Phase Transition of Sink Tokens.}
We investigate which tokens serve as sinks across layers and whether this varies by architecture. Figure~\ref{fig:sink_distribution} shows the attention inflow, defined by the token receiving the highest attention mass per input, for DINOv2-L, DINOv2-G \citep{oquab2023dinov2}, and OpenCLIP-L-16 \citep{ilharco2021openclip}. All three models exhibit a similar transition. In early layers, the \texttt{[CLS]} token acts as the primary sink. In deeper layers, this role shifts to patch tokens: \texttt{[CLS]} inflow decreases while patch sinks emerge. We now ask whether sink behavior is distributed across all attention heads or concentrated in specialized ones.

\paragraph{Head Specialization.}
We map the frequency of sink behavior across layers and heads in DINOv2-L. Figure~\ref{fig:head_specificity} reveals a sparse, vertical structure: certain heads act as sinks for nearly $80\%$ of inputs, while adjacent heads never do. This indicates that sink mechanisms are head specific. Given this specialization, we return to our main question: do these dedicated sink heads implement \nop, broadcast, or both?

\begin{wrapfigure}{r}{0.4\linewidth}
\centering
\vspace{-6mm}
\includegraphics[width=0.95\linewidth]{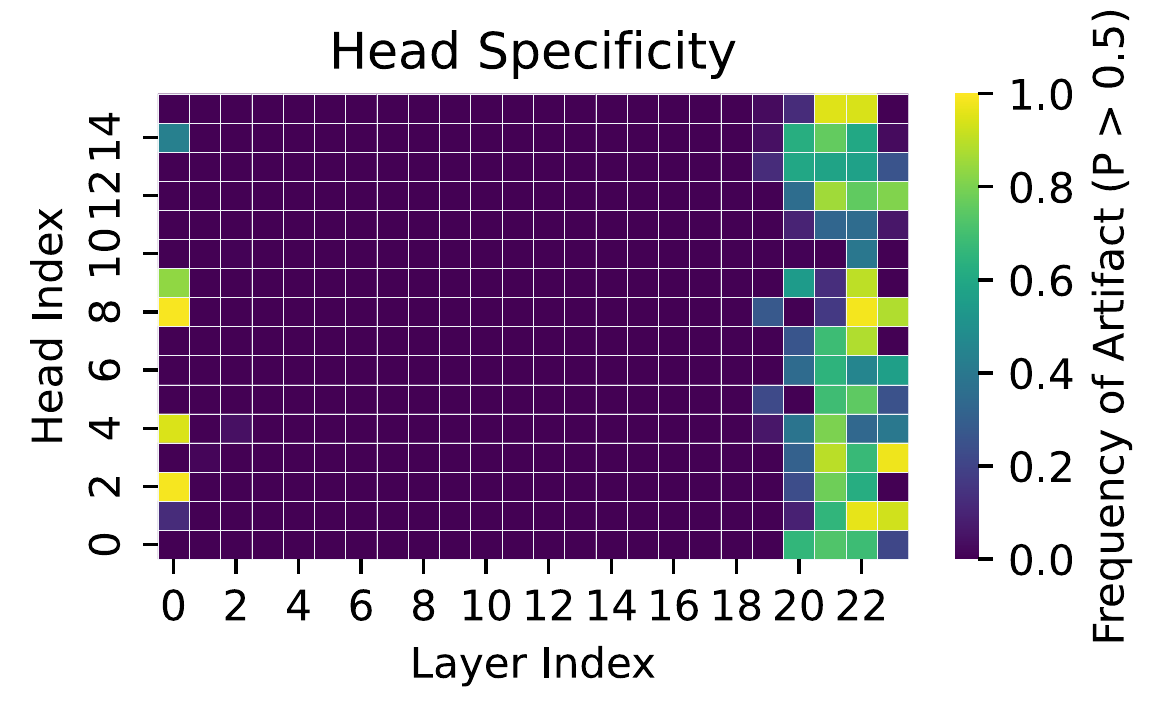}
\vspace{-3mm}
\caption{
\textbf{Head specialization in DINOv2-L.}
Sink behavior is concentrated in a sparse subset of attention heads across layers.
}
\label{fig:head_specificity}
\vspace{-5mm}
\end{wrapfigure}

\paragraph{Distinct Functional Regimes.}
To answer this, we project every detected sink in DINOv2-G onto two diagnostic axes: the value norm ratio $\|\vv_s\|/\|\vv_i\|$ and the stable rank of the attention update. Our main empirical finding is not the existence of either regime in isolation, but that these regimes can be separated and shown to coexist within the same pretrained ViT. Figure~\ref{fig:dino_vitl_sinks} reveals two distinct modes.

The first cluster, concentrated in early-to-middle layers, exhibits strongly suppressed value norms ($<0.2\times$ the baseline norm), consistent with adaptive \nop: these heads nullify their contribution to the residual stream. The second cluster, concentrated in deeper layers, shows standard value norms paired with approximately rank-1 updates, consistent with broadcast: these heads produce non-negligible shared residual contributions across the sequence. This dual phenomenology also appears in other ViTs (Appendix~\ref{app:additional_results}).

\paragraph{Validation via Registers.}
Register-equipped DINOv2-G provides an additional test of this distinction. Registers absorb most sink mass from patch and \texttt{[CLS]} tokens, but the resulting register sinks retain the same two regimes: most have suppressed value norms consistent with \nop, while a smaller group has content-scale values and approximately rank-1 updates consistent with broadcast (Appendix~\ref{app:additional_results}). Thus, registers relocate rather than eliminate sink computation: they are repurposed for \nop{} and, because they remain sequence positions, can themselves become broadcast sinks.

\begin{figure}[h]
\centering
\includegraphics[width=\linewidth]{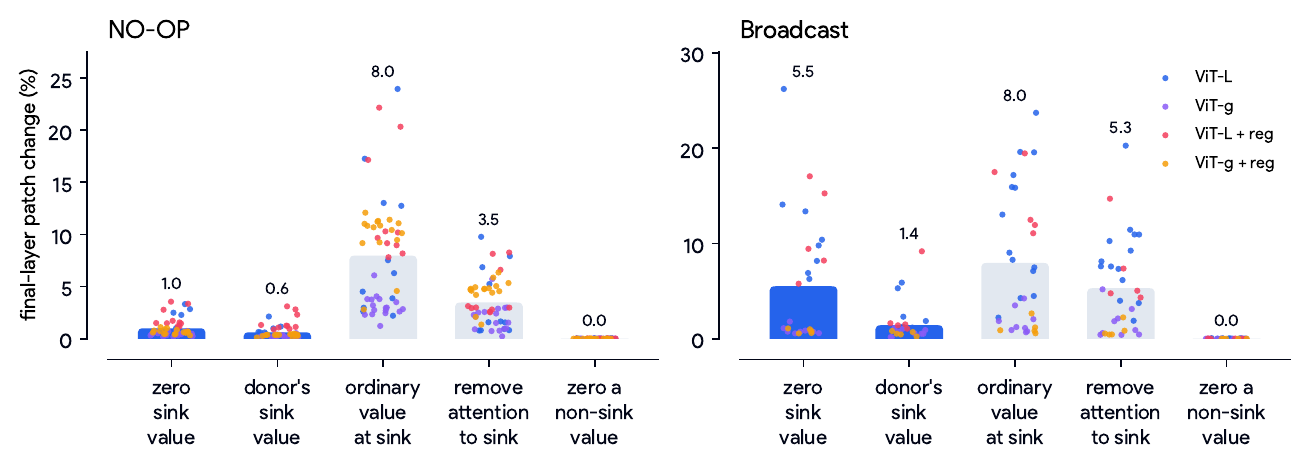}
\caption{
\textbf{Causal validation of sink function in pretrained ViTs.}
For detected \nop-like and broadcast-like sinks, we intervene on the sink value or its incoming attention and measure the resulting change in final-layer patch representations. \nop-like sinks are insensitive to zeroing or replacing the sink value with other sink values, but become disruptive when assigned an ordinary value or when attention to the sink is removed. Broadcast-like sinks are sensitive to removing their contribution, while replacement by another sink value has a substantially smaller effect than replacement by an ordinary value. Zeroing a non-sink value provides a control.
}
\label{fig:def1_interventions}
\end{figure}

\paragraph{Causal Validation in Pretrained ViTs.}
We next test whether these observational signatures correspond to distinct causal roles (Figure~\ref{fig:def1_interventions}). For \nop-like sinks, zeroing the sink value, replacing it with the head's mean sink value, or substituting a sink value from another image produces little change in final-layer patch representations. In contrast, replacing the sink with an ordinary value or removing attention to it is substantially more disruptive. This supports the interpretation of the sink as a near-null destination that suppresses the head's residual update. Broadcast-like sinks show the complementary pattern: zeroing their value or removing attention to them substantially alters downstream representations, confirming that they contribute a non-negligible shared update. However, replacing the sink value with the head's mean or with a donor image's sink value produces much smaller changes than inserting an ordinary value, suggesting that the shared contribution is largely image-invariant in the tested heads. These interventions therefore causally distinguish near-null suppression from nonzero shared residual contributions while qualifying the interpretation of broadcast as necessarily carrying image-specific information.

These results motivate studying architectural interventions for the two computations. Gating provides a direct architectural implementation of conditional suppression for adaptive \nop, while registers relocate sink computation without removing the single-position routing pattern. In the next section, we introduce an explicit position-free global pathway for shared communication and study how these interventions reorganize sink behavior when combined.

\section{Architectural Interventions for Sink Computation}
\label{sec:mitigation}

Our analysis motivates complementary architectural interventions for the two sink mechanisms. Gating gives adaptive-\nop{} heads a direct mechanism for suppressing their residual updates, while an explicit global pathway provides a position-free route for shared communication. Registers provide dedicated sequence positions that can absorb and relocate sink computation (results with registers are reported in Tables \ref{tab:lejepa_imagenet} and \ref{tab:lejepa_dense_prediction_app} of Appendix~\ref{app:mitigation_exp}).

\paragraph{An Explicit Global Pathway.}
Given pre-normalized token representations $\widetilde{\mX}^{\ell}\in\mathbb{R}^{n\times d}$, we construct a permutation-invariant global state and return it to each token through a learned gate:
\begin{equation}
\label{eq:global_pathway}
\begin{aligned}
\pi_j^\ell
&=
\frac{
\exp\!\left((\bm{w}_P^\ell)^\top\widetilde{\vx}_j^\ell\right)
}{
\sum_{k=1}^{n}
\exp\!\left((\bm{w}_P^\ell)^\top\widetilde{\vx}_k^\ell\right)
},
&
\bm{g}^\ell
&=
\sum_{j=1}^{n}
\pi_j^\ell\mW_{V,G}^\ell\widetilde{\vx}_j^\ell,
\\
\gamma_i^\ell
&=
\sigma\!\left(
(\widetilde{\vx}_i^\ell)^\top\bm{w}_\gamma^\ell+b_\gamma^\ell
\right),
&
\vx_i^{\ell+1}
&=
\vx_i^\ell+\vo_i^{\ell,\mathrm{attn}}
+\gamma_i^\ell\mW_{O,G}^\ell\bm{g}^\ell.
\end{aligned}
\end{equation}
Unlike a broadcast sink, $\bm{g}^\ell$ is not associated with a sequence position and therefore cannot attract an attention column. The pathway nevertheless preserves the characteristic shared update: stacking $\gamma_i^\ell$ across tokens produces an update matrix of rank at most one.

\paragraph{Encouraging Broadcast Relocation.}
The global pathway provides an explicit route for shared communication, while the task objective alone permits either branch to carry the shared update. We therefore encourage relocation from attention to the global pathway by penalizing the common component of the attention update. Let $\mO_b^{\ell,\mathrm{attn}}$ denote the output-projected attention update for example $b$, measured before any output gate, and define its common component by
$\bar{\vo}_b^{\ell,\mathrm{attn}}
=n^{-1}\sum_i\vo_{bi}^{\ell,\mathrm{attn}}$.
We train with
% Gradients flow through both numerator and denominator.
\begin{equation}
\label{eq:broadcast_penalty}
\mathcal{L}
=
\mathcal{L}_{\mathrm{task}}
+
\lambda_B
\frac{1}{|\mathcal{L}_B|}
\sum_{\ell\in\mathcal{L}_B}
\frac{
\mathcal{E}\!\left(\bar{\mO}^{\ell,\mathrm{attn}}\right)
}{
\mathcal{E}\!\left(\mO^{\ell,\mathrm{attn}}\right)+\epsilon
},
\qquad
\mathcal{E}(\mZ)=|\mZ|^{-1}\|\mZ\|_F^2,
\end{equation}
where $\bar{\mO}^{\ell,\mathrm{attn}}$ stacks the batchwise common components and $\mathcal{L}_B$ denotes the penalized layers. Up to the stabilizer $\epsilon$, the ratio measures the fraction of the attention update's energy shared across tokens and is invariant to rescaling the attention branch. Penalizing all attention layers encourages shared computation to use the global pathway rather than move to another attention block. The penalty acts on the shared component carried by attention rather than on attention concentration itself.

\begin{wrapfigure}{r}{0.68\linewidth}
\vspace{-2mm}
\centering
\includegraphics[width=\linewidth]{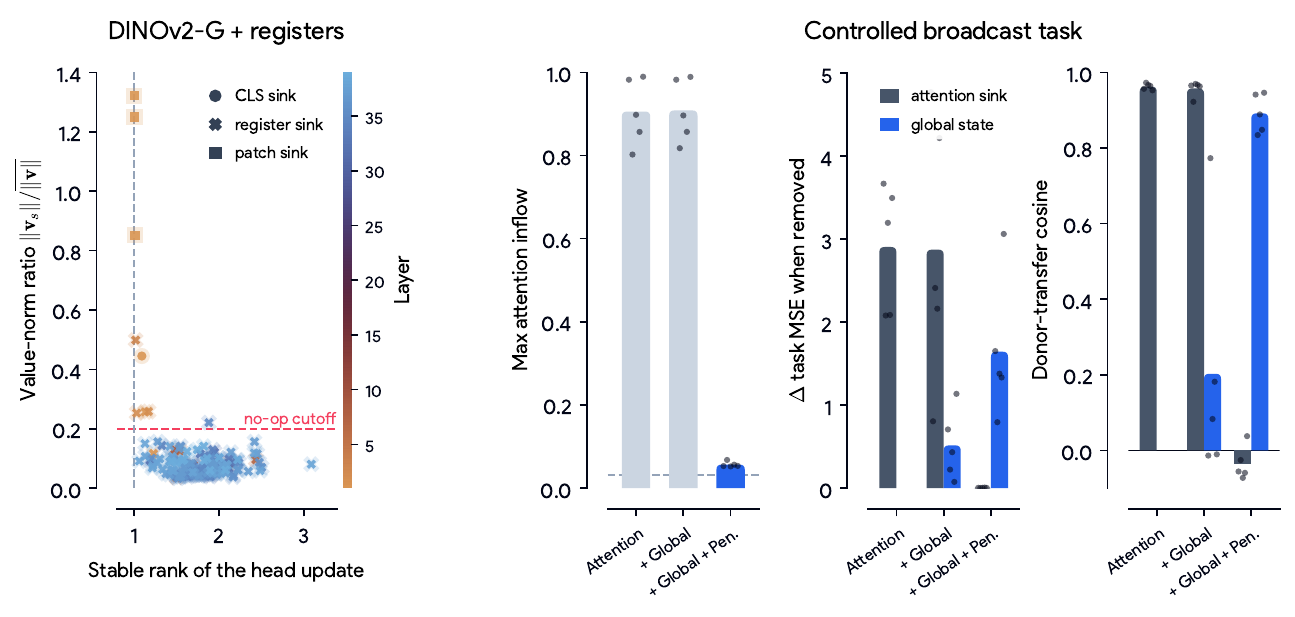}
\vspace{-2mm}
\caption{
\textbf{Diagnostics and architectural reorganization.}
\textbf{(A)} DINOv2-G + registers sinks separate into \nop-like low-value sinks and broadcast-like rank-1 sinks.
\textbf{(B)} On the controlled task, adding the global pathway alone leaves broadcast in attention, whereas the common-mode penalty relocates the shared computation to the global branch.
%\textbf{(C)} In ViT-L, gating removes detected \nop{} sinks, the penalized global pathway reduces broadcast sinks, and their combination reduces, but does not eliminate, the elevated broadcast-sink frequency observed with gating alone.
}
\vspace{-3mm}
\label{fig:dino_vitl_sinks}
\end{wrapfigure}

\paragraph{Experiments.}
We first evaluate the intervention on the two-layer broadcast task from Section~\ref{sec:broadcast}, comparing attention alone, the global pathway without the penalty, and the penalized global pathway using matched training batches across five seeds. We measure task error, maximum attention inflow, and causal dependence on each branch through attention ablation, global-state ablation, and cross-example global-state replacement. Appendix \ref{app:global-pathway-causal} additionally reports a controlled evaluation with uniform recipient weights, $\gamma_i^\ell=1$, to isolate shared communication.

We then train ViT-L LeJEPA \citep{balestriero2025lejepa} for 50 epochs with four variants: a vanilla baseline, Gated Attention (`G1' from \citet{qiu2025gated}), the penalized global pathway, and their combination. The broadcast penalty uses $\lambda_B=1$ on both controlled-task layers and $\lambda_B=0.1$ on ViT layers 1--23, leaving block 0 unpenalized. We classify detected sinks using the diagnostics from Section~\ref{sec:evidence}, with the counting rule of Appendix~\ref{app:additional_results}. ImageNet-1k classification and frozen-backbone ADE20K and Pascal VOC 2012 probes provide complementary representation-quality controls. Full configurations, probe settings and register comparisons appear in Appendix~\ref{app:mitigation_exp}.

\paragraph{Controlled-Task Results.}
% Five-seed results: toy_gp/results/RESULTS.md,
% arms base / gp / gp_pen_pl12_frac.
On the controlled task, adding the pathway alone leaves attention free to carry the broadcast signal: second-layer maximum inflow remains $0.91$, matching the attention baseline. Adding the penalty reduces inflow to $0.06$ in both layers while preserving similar task MSE ($0.144$ versus $0.140$). Attention ablation then has little effect, whereas global-state ablation removes the shared update; donor replacement yields an output change aligned with the desired donor-minus-recipient payload (cosine similarity $0.89$). The uniform-recipient evaluation in Appendix~\ref{app:global-pathway-causal} independently confirms that causal dependence and source selection move from attention to the global branch. Thus, the pathway can carry the broadcast content without requiring any sequence position to attract attention from all recipients.

\paragraph{Architectural Interactions in ViT-L.}
% Sinks per image at the epoch-50 checkpoints of Table 1 (cell 4 = hierarchical gate), 1984 ADE20K images,
% layers 1--23, relaxed rule (>= 90% of queries give >= 0.5 to one token), raw value-norm split.
% Source: appendix_vit/gp_analysis/results/<tag>_ep49_ade/per_image.npz; literal-Definition-1 counts in tab:lejepa_eps_sensitivity.
In ViT-L, the interventions reorganize the two sink regimes differently. Gating reduces the frequency of detected \nop{} sinks from $0.25$ to $0.00$ per image, while the penalized global pathway reduces broadcast sinks from $0.09$ to $0.04$ per image. Gating also increases broadcast-sink frequency to $3.80$ per image. Adding the penalized global pathway reduces this to $0.53$, an approximately $86\%$ reduction relative to gating alone, while keeping \nop{} sinks near zero ($0.03$ per image; Table~\ref{tab:lejepa_sink_counts}). Thus, making suppression explicit shifts shared computation toward broadcast-like sinks, while the global pathway redirects much of this shift without eliminating broadcast-like sinks altogether. % These observations establish the interaction, without identifying why gating increases broadcast-sink frequency.

\begin{table}[t]
% \vspace{-0mm}
\small
\centering
\setlength{\tabcolsep}{3pt}
\caption{
\textbf{Linear probing performance.}
LeJEPA ViT-L models trained for 50 epochs. ImageNet: top-1 of an online linear probe. ADE20K and VOC12: mIoU of linear segmentation probes on frozen features.
}
\begin{tabular}{
l
r@{${}\pm{}$}l
r@{${}\pm{}$}l
r@{${}\pm{}$}l
}
\toprule
Variant & \multicolumn{2}{c}{ImageNet} & \multicolumn{2}{c}{ADE20K} & \multicolumn{2}{c}{VOC12} \\
\midrule
Baseline
& \multicolumn{2}{c}{$0.602$}
& $0.212$ & $0.001$
& $0.601$ & $0.0005$ \\

Gating
& \multicolumn{2}{c}{$\mathbf{0.610}$}
& $0.216$ & $0.001$
& $0.595$ & $0.0004$ \\

Global + Penalty
& \multicolumn{2}{c}{$0.549$}
& $0.181$ & $0.001$
& $0.527$ & $0.0004$ \\

% Cell 4 = hierarchical recipient gate (option B). Option A (shared gate) at the same epoch:
% & \multicolumn{2}{c}{$0.603$} & $0.217$ & $0.001$ & $0.608$ & $0.0003$ \\
Gating + Global + Penalty
& \multicolumn{2}{c}{$0.606$}
& $\mathbf{0.217}$ & $\mathbf{0.0004}$
& $\mathbf{0.610}$ & $\mathbf{0.0001}$ \\
\bottomrule
\end{tabular}
% \vspace{-1mm}
\label{tab:lejepa_full_results}
\vspace{-4mm}
\end{table}

Table~\ref{tab:lejepa_full_results} separates sink prevalence from representation quality. The penalized global pathway alone suppresses broadcast-like sinks most strongly but performs worse than the baseline on all three probes. In contrast, combining gating with the global pathway retains some broadcast-like sinks while achieving near-baseline ImageNet accuracy and the strongest dense-probe results among the tested variants. Together with the controlled interventions, these results show that architectural changes can relocate and reorganize sink computation without making sink prevalence itself a proxy for representation quality. Effective intervention should therefore be evaluated by both how computation is redistributed and the representations that result.

\section{Discussion}

Our results show that attention sinks in ViTs should not be treated as a single phenomenon. Value norms and update rank distinguish two canonical regimes, while causal interventions connect them to near-null suppression and nonzero shared residual contributions. Registers further show that sink computation can be relocated rather than eliminated. The mitigation experiments extend this picture: architectural changes reorganize these computations rather than acting as independent fixes. Gating removes detected \nop-like sinks but induces substantially more broadcast-like sinks, while adding the global pathway reduces this shift. Thus, changing the mechanisms available for suppression can alter how shared communication is implemented elsewhere in the model.

These interactions also change how sink mitigation should be evaluated. Sink prevalence alone is not a sufficient measure of representation quality: the global-pathway variant suppresses broadcast-like sinks most strongly but performs worse on downstream probes, whereas combining gating with the global pathway gives the strongest dense-probe results among the tested variants despite retaining some broadcast-like sinks. The goal is therefore not simply to eliminate sinks, but to provide architectural alternatives for the computations they support. Our framework is not exhaustive: real models may distribute broadcast computation across layers or use higher-rank shared signals, potentially requiring richer global pathways, and the appropriate pathway design may vary with scale and objective. More broadly, effective intervention depends not only on identifying what a sink is computing, but also on giving the model another way to compute it.

\subsection*{Reproducibility Statement}

We make a preliminary version of our codebase publicly available, including code to run the controlled toy tasks and ViT experiments: \url{https://github.com/KempnerInstitute/attention_sink}.
Additional details on task setups, evaluation protocols, and training configurations are provided in the appendix.

% \section*{Acknowledgements}

% This work has been made possible in part by a gift from the Chan Zuckerberg Initiative Foundation to establish the Kempner Institute for the Study of Natural and Artificial Intelligence at Harvard University. L.F. and M.J. are supported by the Kempner Graduate Fellowship at Harvard University. TF and TAK are supported by the Kempner Institute Research Fellowship.

% \newpage

% \input{causal_interventions}

% \input{broadcast_mitigation}

% \clearpage

\bibliographystyle{iclr2027_conference}
\bibliography{main_corrected}

%%%%%%%%%%%%%%%%%%%%%%%%%%%%%%%%%%%%%%%%%%%%%%%%%%%%%%%%%%%%

\clearpage

\appendix

\etocdepthtag.toc{appendix}

\etocsettagdepth{main_iclr_pathway}{none}
\etocsettagdepth{appendix}{subsection}

\renewcommand{\contentsname}{Appendix Contents}
\tableofcontents
\clearpage

\section{Controlled tasks and causal interventions}

\subsection{Controlled Task Setups}
\label{app:controlled-task-setups}

We use two deliberately minimal synthetic tasks to isolate the mechanisms described in
Section~3: adaptive NOP, which requires the attention branch to conditionally suppress its
update, and broadcast, which requires content-dependent information to be communicated to
all positions. In both settings, inputs are sampled i.i.d. from an isotropic Gaussian
distribution and fresh samples are generated at every optimization step.

\paragraph{Adaptive NOP.}
We implement the NOP task using a minimal single-head Transformer block with no MLP,
mirroring a ViT attention layer with a pre-norm residual connection. Let
$X \in \mathbb{R}^{L \times d}$ denote the residual stream and
$\widetilde{X} = \operatorname{LN}(X)$ the normalized states used to form queries,
keys, and values. The block computes
\begin{equation}
    \A = \operatorname{softmax}\!\left(
        \frac{QK^\top}{\sqrt{d_h}}
    \right),
    \qquad
    \Delta O = W_O(\A V),
    \qquad
    O = X + \Delta O,
\end{equation}
where $Q$, $K$, and $V$ are obtained from $\widetilde{X}$ using learned linear maps
without bias. We use model width $d=64$, head dimension $d_h=64$, and sequence length
$L=32$, with learned positional embeddings and a learned BOS offset added to token $0$
before LayerNorm. Training examples
$X_{\mathrm{in}} \in \mathbb{R}^{B \times L \times d}$ are sampled i.i.d. from a
standard Gaussian distribution.

To determine which tokens require a NOP, we sample once a fixed unit vector
$v \in \mathbb{R}^d$ and compute
$s_i = \langle (X_{\mathrm{in}})_i, v \rangle$. We define
$m_i = \mathbbm{1}[s_i > \lambda]$ and impose the regression target directly on the
attention branch,
\begin{equation}
    \Delta y_i
    =
    (1-m_i)(X_{\mathrm{in}})_i
    +
    m_i \gamma (X_{\mathrm{in}})_i .
    \label{eq:nop-toy-target}
\end{equation}
Thus, $\gamma=0$ requires a zero attention update for masked tokens, leaving the
residual stream unchanged, while $\gamma>0$ yields a partially damped condition.
We minimize mean-squared error
$\|\Delta O-\Delta Y\|_2^2$ using AdamW with learning rate
$3 \cdot 10^{-3}$ and weight decay $10^{-4}$. Unless stated otherwise, models are
trained for $30$k steps with batch size $B=1024$; the optimization sweeps in
Appendix \ref{app:supplementary-controlled} use $B=2048$.

For the causal-validation experiment in Appendix~A.3, we use a closely related
block-output formulation that makes the NOP interpretation explicit. Some query
tokens have identity targets, $y_i=x_i$, and therefore require the attention branch
to contribute an approximately zero residual update, whereas the remaining queries
have zero-output targets and require an active update. Token $0$ is provided as a
dedicated, input-independent token and is excluded from the task loss, allowing the
model to use it as an attention destination without requiring it to encode
token-specific task information. The resulting trained model is analyzed using the
causal interventions defined in Appendix \ref{app:causal-interventions}.

\paragraph{Broadcast.}
We implement the broadcast task using a two-layer ViT-style Transformer with one
attention head per layer and pre-norm residual connections. Each attention layer
applies LayerNorm, computes $Q$, $K$, and $V$ using learned linear maps without bias,
forms
\begin{equation}
    \A = \operatorname{softmax}\!\left(
        \frac{QK^\top}{\sqrt{d_h}}
    \right),
\end{equation}
and writes the attention output back through a learned output projection,
$X_{\mathrm{out}} = X_{\mathrm{in}} + W_O(\A V)$. Each attention layer is followed
by an MLP block
$\operatorname{LN} \rightarrow \operatorname{Linear}(d,4d)
\rightarrow \operatorname{GELU}
\rightarrow \operatorname{Linear}(4d,d)$
with a residual connection. We use $d=d_h=64$, sequence length $L=32$, and learned
positional embeddings added only before the first attention layer.

For each sequence, we sample a fixed selector direction $u$ once and choose the
content-dependent source position
\begin{equation}
    j^\star = \arg\max_j \langle x_j, u \rangle .
    \label{eq:broadcast-source}
\end{equation}
Every recipient token must then add a transformed version of the selected source,
\begin{equation}
    y_i = x_i + \alpha \bigl(x_{j^\star} R\bigr),
    \label{eq:broadcast-toy-target}
\end{equation}
where $R \in \mathbb{R}^{d \times d}$ is sampled once and fixed throughout training.
We construct $R$ as an approximately orthogonal matrix using QR factorization and
typically set the broadcast strength to $\alpha=1$. This objective requires the model
to identify a source whose position varies across examples and communicate its
content to every token in the sequence.

We train on i.i.d. Gaussian sequences
$X \in \mathbb{R}^{B \times L \times d}$ with batch size $B=1024$, minimizing
mean-squared error between the model outputs and the broadcast targets. Optimization
uses AdamW for $30$k steps with learning rate $5 \cdot 10^{-4}$ and weight decay
$10^{-4}$. The diagnostics and causal interventions used to determine whether the
learned attention sinks implement the intended communication mechanism are defined
in Appendix \ref{app:causal-interventions}.

\subsection{Causal Interventions and Evaluation Metrics}
\label{app:causal-interventions}

We use both observational diagnostics and targeted interventions to determine whether a
learned sink implements adaptive NOP or broadcast. The key distinction in the causal
analysis is between interventions on the \emph{payload} carried by a high-attention
destination, which preserve the learned attention pattern, and interventions on the
\emph{routing} itself, which alter where attention is directed.

\paragraph{Sink diagnostics.}
For an attention tensor
$\A \in \mathbb{R}^{B \times L \times L}$, we measure the average attention received
by key position $j$ as
\begin{equation}
    \bar{a}_j
    =
    \frac{1}{BL}
    \sum_{b=1}^{B}
    \sum_{i=1}^{L}
    \A_{bij},
\end{equation}
and define the sink score as $\max_j \bar{a}_j$. This measure is appropriate when the
sink occupies a fixed position, as in the adaptive-NOP task.

For broadcast, the source position is content-dependent and therefore varies across
examples. For each example $b$, we identify the dominant attention column
\begin{equation}
    \hat{j}_b
    =
    \arg\max_j
    \frac{1}{L}
    \sum_{i=1}^{L}
    \A_{bij},
\end{equation}
and report the fraction of examples for which
$\hat{j}_b = j_b^\star$, where $j_b^\star$ is the ground-truth source defined in
Eq.~\ref{eq:broadcast-source}. We additionally report the mean attention assigned to
the selected source,
\begin{equation}
    a_{\mathrm{src}}
    =
    \frac{1}{BL}
    \sum_{b=1}^{B}
    \sum_{i=1}^{L}
    \A_{bi j_b^\star}.
\end{equation}
For the adaptive-NOP experiment, we analogously measure attention to the dedicated
token $0$ separately for NOP and active queries. We also compare the norm of the
output-projected sink value, $\|W_Ov_s\|$, with that of ordinary values. For broadcast,
we additionally inspect the selected-token value and key norms and the rank structure
of the attention update, as described in Section~3.2.

\paragraph{Value interventions.}
Let $s$ denote the sink position under intervention: $s=0$ in the adaptive-NOP task
and $s=j^\star$ in the broadcast task. To isolate the causal role of the value carried
at this destination, we replace $v_s$ while preserving the attention matrix $\A$.
If $v_s^{\mathrm{patch}}$ denotes the patched value, the resulting local change in the
output of the attention branch is
\begin{equation}
    \Delta o_i
    =
    \A_{is} W_O
    \left(
        v_s^{\mathrm{patch}} - v_s^{\mathrm{original}}
    \right).
    \label{eq:value-patching-local}
\end{equation}
We verify this local prediction at the intervened layer before propagating the modified
activation through the remainder of the network.

We consider three main value interventions. First, \emph{zero ablation} sets the sink
value to zero. Second, \emph{donor-sink patching} replaces it with the corresponding
sink value from an independently sampled donor example. In the broadcast task, this
means replacing the recipient value at $j^\star$ with the value at the donor's own
selected source position $j^\star_{\mathrm{donor}}$. Third, as a content control,
we insert an ordinary donor non-sink value at the same high-attention destination.
For broadcast we also zero an ordinary non-sink value while leaving the selected
source unchanged. These controls separate the effect of attending strongly to a
position from the specific information contained in its value.

\paragraph{Routing intervention.}
As a complementary intervention, we remove attention to the sink while leaving the
values unchanged. For each query $i$, we set the sink weight to zero and renormalize
the remaining attention weights,
\begin{equation}
    \widetilde{\A}_{ij}
    =
    \begin{cases}
        0, & j=s, \\[3pt]
        \displaystyle\frac{\A_{ij}}{1-\A_{is}}, & j\neq s,
    \end{cases}
    \label{eq:sink-routing-ablation}
\end{equation}
whenever $\A_{is}<1$. This intervention tests whether routing through the sink is
necessary, whereas value patching tests what information is carried at the selected
destination. Because removing a dominant attention column can also disrupt other
learned computations, we interpret routing and zero-ablation experiments as
necessity tests rather than as selective removal of the hypothesized payload.

\paragraph{Task error and intervention magnitude.}
For both tasks, we report mean-squared error between the model output and the
corresponding target. To quantify the magnitude of an intervention, let
$Y$ and $\widetilde{Y}$ denote the original and intervened model outputs and define
$\Delta Y = \widetilde{Y}-Y$. We report the root-mean-square output change
\begin{equation}
    \operatorname{RMS}(\Delta Y; I)
    =
    \sqrt{
        \frac{1}{|I|d}
        \sum_{i\in I}
        \|\Delta y_i\|_2^2
    },
    \label{eq:intervention-rms}
\end{equation}
where $I$ denotes the evaluated query set. For adaptive NOP, this quantity is reported
separately for identity-target and active queries, allowing us to test whether
interventions selectively affect the tokens that route to the NOP sink.

\paragraph{Broadcast coherence and payload transfer.}
For broadcast interventions, we additionally ask whether the induced change is shared
across recipient tokens and whether it carries the \emph{correct} donor content.
Let
\begin{equation}
    \overline{\Delta y}
    =
    \frac{1}{L}
    \sum_{i=1}^{L}
    \Delta y_i
\end{equation}
denote the token-wise shared component of the intervention-induced output change.
We quantify the fraction of intervention energy contained in this shared component as
\begin{equation}
    E_{\mathrm{shared}}
    =
    \frac{
        \|\overline{\Delta y}\|_2^2
    }{
        \frac{1}{L}\sum_{i=1}^{L}\|\Delta y_i\|_2^2
    }.
    \label{eq:shared-component-energy}
\end{equation}
Values near one indicate that the intervention produces an approximately common
change across recipient positions.

For a recipient example $r$ and independent donor example $d$, the desired change in
broadcast payload is
\begin{equation}
    \Delta p^{r\leftarrow d}
    =
    \alpha
    \left(
        x_{j_d^\star}^{(d)}R
        -
        x_{j_r^\star}^{(r)}R
    \right).
    \label{eq:desired-donor-payload}
\end{equation}
We measure payload transfer by the cosine similarity
\begin{equation}
    \operatorname{Align}
    =
    \cos\!\left(
        \overline{\Delta y},
        \Delta p^{r\leftarrow d}
    \right).
    \label{eq:donor-alignment}
\end{equation}
This distinguishes a coherent but arbitrary perturbation from transfer of the
content selected in the donor sequence.

Finally, we construct the corresponding donor-patched target while preserving the
recipient's local token content,
\begin{equation}
    y_i^{r\leftarrow d}
    =
    x_i^{(r)}
    +
    \alpha x_{j_d^\star}^{(d)}R,
    \label{eq:donor-target}
\end{equation}
and report the MSE between the intervened model output and
$Y^{r\leftarrow d}$. A reduction in this donor-target error, together with high
payload alignment, provides evidence that the patched value causally carries the
broadcast content rather than merely inducing a large or spatially coherent change.

\subsection{Adaptive NOP: Results and Causal Validation}
\label{app:adaptive-nop-causal}

We next test whether the sink learned on the adaptive-NOP task causally implements
a null operation, rather than merely correlating with one. We use the block-output
variant described in Appendix~\ref{app:controlled-task-setups}, in which identity-target
queries require the attention branch to perform a NOP, while zero-output queries
require an active update. All interventions follow the protocol in
Appendix~\ref{app:causal-interventions}. We evaluate over 64 batches of 256 freshly
sampled sequences and exclude query $0$ from the groupwise analysis, since token $0$
is itself the dedicated sink position.

\paragraph{Adaptive routing to a near-null value.}
The trained model develops a sharply conditional routing pattern. Identity-target
queries assign $99.29\%$ of their attention to the dedicated sink at token $0$,
compared with only $1.22\%$ for zero-output queries; uniform attention would assign
$1/32 \approx 3.13\%$. Thus, the sink is selected specifically when the attention
branch should contribute approximately nothing to the residual stream.

The selected destination also carries an unusually small payload. After the output
projection, the mean sink-value norm is
\[
    \|W_O v_s\| = 0.063,
\]
compared with $7.872$ for ordinary values, making the sink payload approximately
$125\times$ smaller. The combination of selective high attention and a near-null
value is therefore consistent with the adaptive-NOP mechanism predicted in
Section~3.1.

\paragraph{Causal interventions.}
Table~\ref{tab:nop-causal-interventions} summarizes the causal interventions.
Zeroing the already-small sink value leaves task performance essentially unchanged:
MSE changes from $0.01869$ to $0.01867$, while the RMS output change is only
$0.00785$ on adaptive-NOP queries and $0.00065$ on zero-output queries.
Cross-example replacement with another sink value leaves the output unchanged.

In contrast, replacing the value at the same high-attention destination with an
ordinary donor non-sink value increases task MSE to $0.50009$. The induced output
change is strongly concentrated on the queries that route to the sink: RMS change
is $0.97917$ for adaptive-NOP queries, compared with $0.08053$ for zero-output
queries. Thus, the large attention weights to token $0$ are fully capable of
transmitting a substantial update when a non-null payload is placed at that
destination.

Removing token $0$ as an attention destination and renormalizing the remaining
attention weights provides a complementary routing intervention. This increases
task MSE to $0.14391$ and changes adaptive-NOP outputs by $0.49952$ RMS, compared
with $0.08115$ for zero-output queries. Hence, forcing queries that normally use
the sink to attend to ordinary values disrupts the NOP computation.

\begin{table}[h]
    \centering
    \small
    \caption{\textbf{Causal interventions on the adaptive-NOP sink.}
    Mean results over 64 evaluation batches. Value interventions preserve the
    learned attention matrix, whereas the routing intervention removes the sink
    destination and renormalizes the remaining attention weights. Output changes
    are measured relative to the unmodified model and separately for identity-target
    (adaptive-NOP) and zero-output queries.}
    \begin{tabular}{lccc}
        \toprule
        Intervention
        & Task MSE
        & $\Delta$ RMS: NOP
        & $\Delta$ RMS: zero-output \\
        \midrule
        Baseline                  & 0.018694 & 0.000000 & 0.000000 \\
        Zero sink value           & 0.018669 & 0.007849 & 0.000645 \\
        Donor sink value          & 0.018694 & 0.000000 & 0.000000 \\
        Donor non-sink value      & 0.500093 & 0.979167 & 0.080533 \\
        Redirect sink attention   & 0.143905 & 0.499524 & 0.081149 \\
        \bottomrule
    \end{tabular}
    \label{tab:nop-causal-interventions}
\end{table}

\begin{figure}[t]
    \centering
    \includegraphics[width=\linewidth]{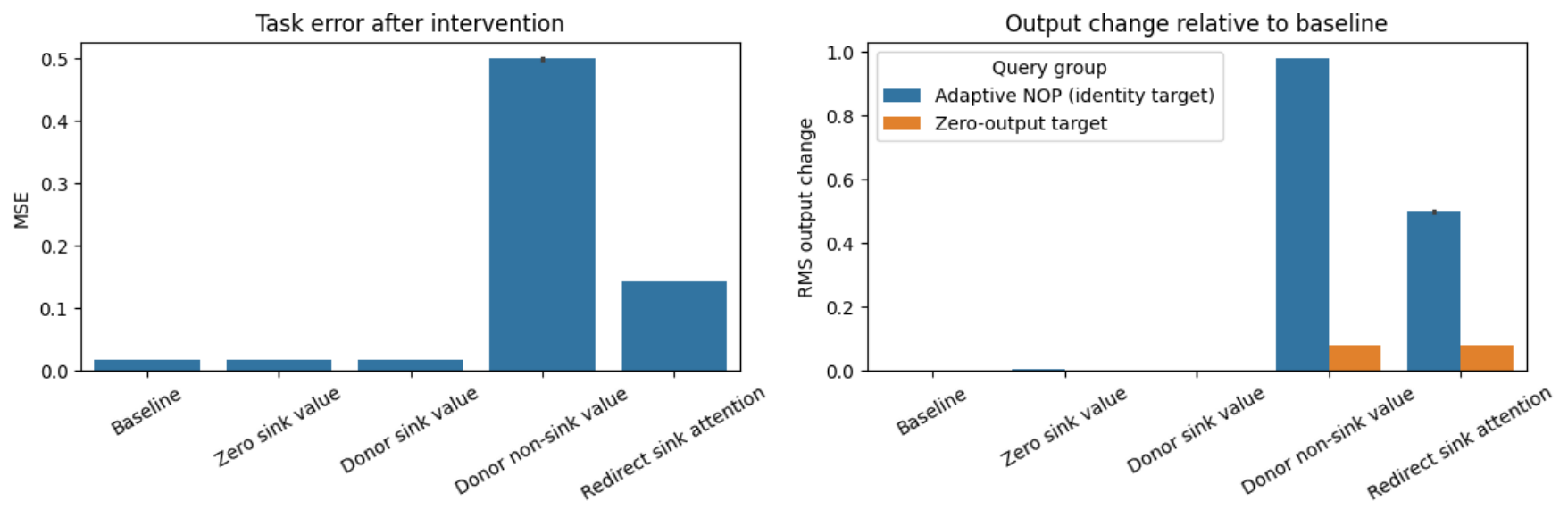}
    \caption{\textbf{Causal validation of the adaptive-NOP mechanism.}
    Left: task error after intervening on the sink value or attention routing.
    Right: RMS output change relative to the baseline, separated by query type.
    Zeroing or exchanging the near-null sink value has little effect, whereas
    inserting an ordinary value at the same destination or redirecting attention
    away from the sink selectively disrupts identity-target queries. Error bars
    denote standard errors over evaluation batches.}
    \label{fig:nop-causal-intervention}
\end{figure}

\paragraph{Interpretation.}
These interventions separate the effect of \emph{where} the head attends from
\emph{what} is stored at that destination. High attention to token $0$ is not
ineffectual: replacing its value with an ordinary payload causes a large output
change. Instead, the baseline update is negligible because the selected destination
carries a near-null value. Conversely, removing access to that destination forces
NOP queries to mix ordinary values and substantially degrades the task. In this
controlled setting, the sink therefore acts as a learned null destination that the
head selects conditionally to suppress its residual update.

\subsection{Broadcast: Results and Causal Validation}
\label{app:broadcast-causal}

We next test whether the dynamic sinks learned on the broadcast task causally carry
the shared, content-dependent payload. We use the two-layer model described in
Appendix~\ref{app:controlled-task-setups} and apply the interventions from
Appendix~\ref{app:causal-interventions} separately at each attention layer. Each
intervention is evaluated over 64 batches of 256 independently sampled recipient
and donor sequences. Layer-1 interventions are propagated through the first MLP
and the complete second block, whereas Layer-2 interventions are propagated through
the final MLP.

\paragraph{Dynamic non-null sinks emerge in both layers.}
Both attention layers learn to route strongly to the content-dependent source
$j^\star$. The dominant attention column matches the ground-truth source in
$95.36\%$ of examples in L1 and $97.81\%$ in L2. Mean attention assigned to
$j^\star$ is $92.90\%$ and $96.68\%$, respectively. Unlike the dedicated
adaptive-NOP sink, these destinations vary across examples and carry substantial
content-dependent values. Thus, attention alone identifies both layers as candidate
broadcast channels, but does not determine which layer directly carries the
transferable payload.

\paragraph{Sink routing is causally important.}
We first test whether the selected destinations are necessary for the learned
computation. The unmodified model achieves an MSE of $0.0481$. Zeroing the selected
value increases MSE to $4.3652$ when intervening in L1 and $3.0729$ in L2.
Similarly, removing attention to the selected source and renormalizing the remaining
weights increases MSE to $0.5886$ and $2.2692$, respectively. In contrast, zeroing
a non-sink value has little effect, yielding MSEs of $0.0547$ in L1 and $0.0487$
in L2. The selected sink is therefore causally important in both layers.

Zero ablation is not a selective test of broadcast content, however: removing a
learned value can destroy other information used downstream. We therefore turn to
cross-example patching to ask the stronger question of whether the value at the
selected destination carries the \emph{specific} shared payload required by the
task.

\begin{table}[t]
    \caption{\textbf{Causal interventions on broadcast sinks.}
    Interventions are applied to one attention layer at a time and propagated through
    the remainder of the network. Task MSE is measured against the original broadcast
    target, whereas donor-target MSE is measured against the target obtained by
    replacing the recipient's broadcast payload with the donor's payload. Final
    $\Delta$ RMS measures the intervention-induced output change over recipient
    positions.}
    \centering
    \small
    \begin{tabular}{llrrr}
        \toprule
        Layer & Intervention
        & Task MSE
        & Final $\Delta$ RMS
        & Donor-target MSE \\
        \midrule
        L1 & Baseline                & 0.0481 & 0.0000 & 1.9139 \\
        L1 & Zero sink value         & 4.3652 & 2.0653 & 6.8849 \\
        L1 & Donor sink value        & 0.1021 & 0.2317 & 1.8684 \\
        L1 & Donor non-sink value    & 0.2554 & 0.4410 & 1.8884 \\
        L1 & Zero non-sink value     & 0.0547 & 0.0652 & 1.9240 \\
        L1 & Redirect sink attention & 0.5886 & 0.7248 & 2.5398 \\
        \midrule
        L2 & Baseline                & 0.0481 & 0.0000 & 1.9139 \\
        L2 & Zero sink value         & 3.0729 & 1.7242 & 3.1502 \\
        L2 & Donor sink value        & 1.9489 & 1.3626 & 0.2299 \\
        L2 & Donor non-sink value    & 2.0133 & 1.3861 & 2.1227 \\
        L2 & Zero non-sink value     & 0.0487 & 0.0263 & 1.9148 \\
        L2 & Redirect sink attention & 2.2692 & 1.4644 & 2.3315 \\
        \bottomrule
    \end{tabular}
    \label{tab:broadcast-causal-interventions}
\end{table}

\paragraph{Cross-example patching localizes the transferable payload.}
For each recipient sequence $r$, we replace the value at its selected destination
$j_r^\star$ with the value at the independently sampled donor source
$j_d^\star$, while preserving the recipient's original attention matrix. We first
verify that the immediate change at the intervened attention layer agrees with the
exact local prediction in Eq.~\ref{eq:value-patching-local}, and then propagate the
modified activation through the rest of the model.

The resulting causal effect differs sharply across layers. Patching the L2 sink with
the donor sink induces a large, recipient-wide output change
($\Delta$ RMS $=1.363$), with more than $99.99\%$ of its energy contained in the
component shared across recipient positions. The mean change aligns with the desired
donor-minus-recipient payload difference with cosine similarity $0.944$. Consistent
with genuine payload transfer, the MSE relative to the donor-patched target falls
from $1.914$ for the baseline output to $0.230$ after the intervention, an improvement
of $1.684$.

Importantly, this effect depends on the \emph{content} stored at the selected
destination, not simply on placing an arbitrary value under a high-attention column.
Replacing the same L2 destination with an ordinary donor non-sink value produces a
perturbation of comparable magnitude
($\Delta$ RMS $=1.386$), but its donor-target MSE is $2.123$, worse than the
unmodified baseline. High attention therefore provides the broadcast routing
operation, while the selected value provides the payload that is broadcast.

The corresponding L1 donor-sink intervention produces substantially weaker transfer.
Its final output change has RMS $0.232$ and aligns with the desired payload difference
with cosine similarity only $0.162$. Donor-target MSE decreases only from $1.914$
to $1.868$, an improvement of $0.046$. Thus, although interventions establish that
the L1 sink is causally involved in the computation, its selected value does not
directly carry the transferable broadcast payload observed at the model output.

\begin{figure}[h]
    \centering
    \includegraphics[width=\linewidth]{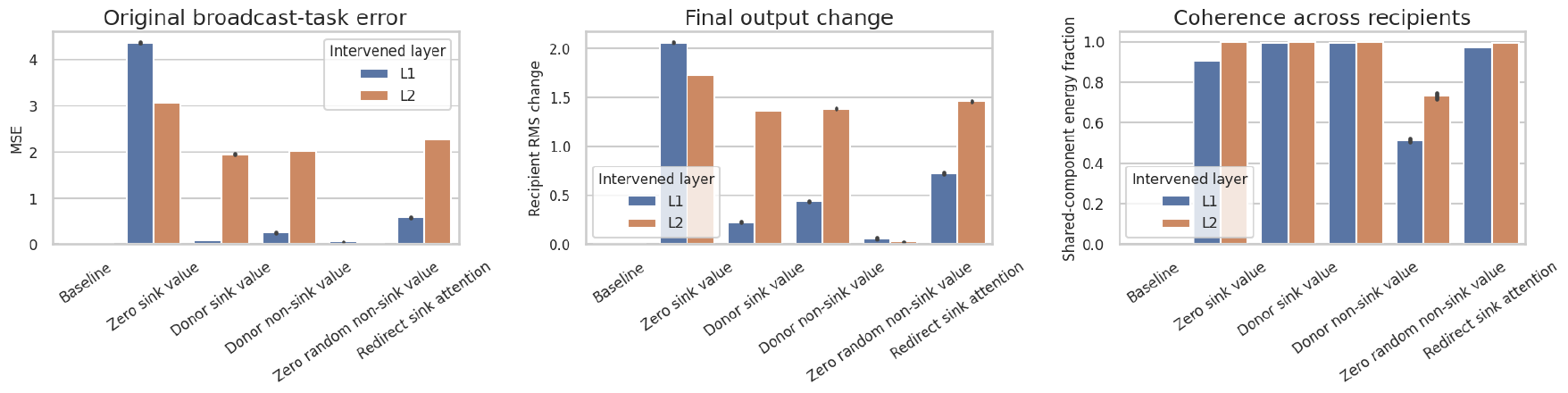}
    \caption{\textbf{Causal interventions on the learned broadcast sinks.}
    Interventions are applied separately in L1 and L2. Left: MSE against the
    original broadcast target. Center: intervention-induced RMS change at the final
    output over recipient positions. Right: fraction of the output-change energy
    explained by the component shared across recipients. Zeroing or redirecting a
    selected sink is substantially more disruptive than perturbing an ordinary
    non-sink position. Error bars denote standard errors over evaluation batches.}
    \label{fig:broadcast-causal-interventions}
\end{figure}

\begin{figure}[h]
    \centering
    \includegraphics[width=\linewidth]{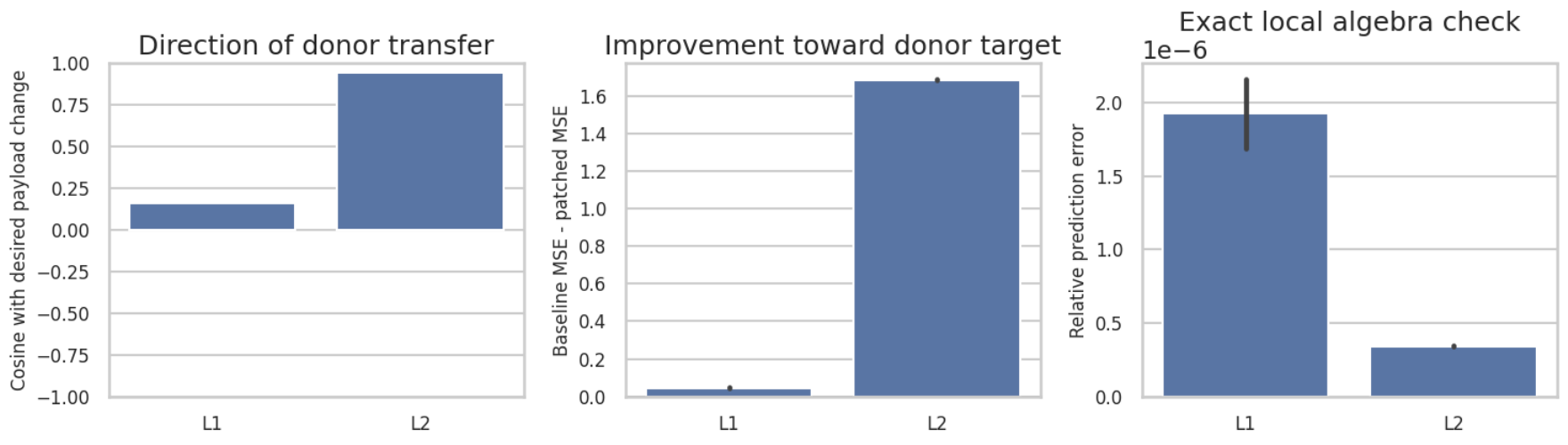}
    \caption{\textbf{Cross-example patching identifies the transferable broadcast
    channel.}
    Replacing the selected value with that of an independent donor produces strong
    donor-payload transfer when applied in L2 but little transfer when applied in
    L1. Left: cosine similarity between the mean recipient-wide output change and
    the desired donor-minus-recipient payload difference. Center: improvement in MSE
    toward the donor-patched target. Right: numerical error in the exact local
    intervention identity. Error bars denote standard errors over evaluation
    batches.}
    \label{fig:broadcast-donor-transfer}
\end{figure}

\paragraph{Interpretation.}
Together, these interventions distinguish three claims that cannot be established
from the attention pattern alone. First, the dynamically selected positions are
causally important: removing their values or redirecting attention away from them
substantially disrupts the task, whereas matched interventions at non-sink positions
do not. Second, the effect of a sink intervention is broadly shared across recipient
tokens, as expected for broadcast. Third, and most importantly, cross-example
patching shows that the L2 sink carries the \emph{specific content} to be broadcast:
changing its value transfers the donor payload while preserving the recipient's
local token content. The broadcast sink is therefore not merely correlated with a
low-rank shared update; in this controlled model, it causally mediates that update.

\subsection{Supplementary Controlled-Task Analyses}
\label{app:supplementary-controlled}

We collect additional analyses of the controlled tasks that complement the causal
experiments in Appendices~\ref{app:adaptive-nop-causal}--\ref{app:broadcast-causal}.
For adaptive NOP, we study how sink formation depends on the frequency and strength
of the NOP condition and examine the spectral signature of the learned routing
mechanism. For broadcast, we analyze layer specialization, source value norms,
query--key geometry, and the rank structure of the attention update. Unless stated
otherwise, task definitions and training configurations follow
Appendix~\ref{app:controlled-task-setups}.

\paragraph{Adaptive-NOP optimization sweeps.}
We use the branch-supervised adaptive-NOP formulation of
Eq.~\ref{eq:nop-toy-target} and track sink formation throughout training using the
sink score defined in Appendix~\ref{app:causal-interventions}. For the sweep
experiments, we use batch size $B=2048$ and log the sink score every 10 optimization
steps.

We consider two complementary sweeps. First, to vary the frequency with which a NOP
is required, we choose the threshold $\lambda$ to obtain target NOP probabilities
$p \in \{0.1,0.3,0.5,0.7,0.9\}$ under the Gaussian input distribution. For a
unit-norm gate vector,
\begin{equation}
    \lambda = \Phi^{-1}(1-p),
\end{equation}
where $\Phi$ is the standard-normal cumulative distribution function. Second, we
hold $\lambda=0$, corresponding to $p=0.5$, and vary the damping coefficient over
\begin{equation}
    \gamma
    \in
    \{0,10^{-3},10^{-2},10^{-1},0.5,1.0\}.
\end{equation}
Here $\gamma=0$ corresponds to an exact NOP target, while larger values progressively
weaken the required suppression; at $\gamma=1$, both branches have the same target.

\begin{figure}[t]
    \centering
    \includegraphics[width=\linewidth]{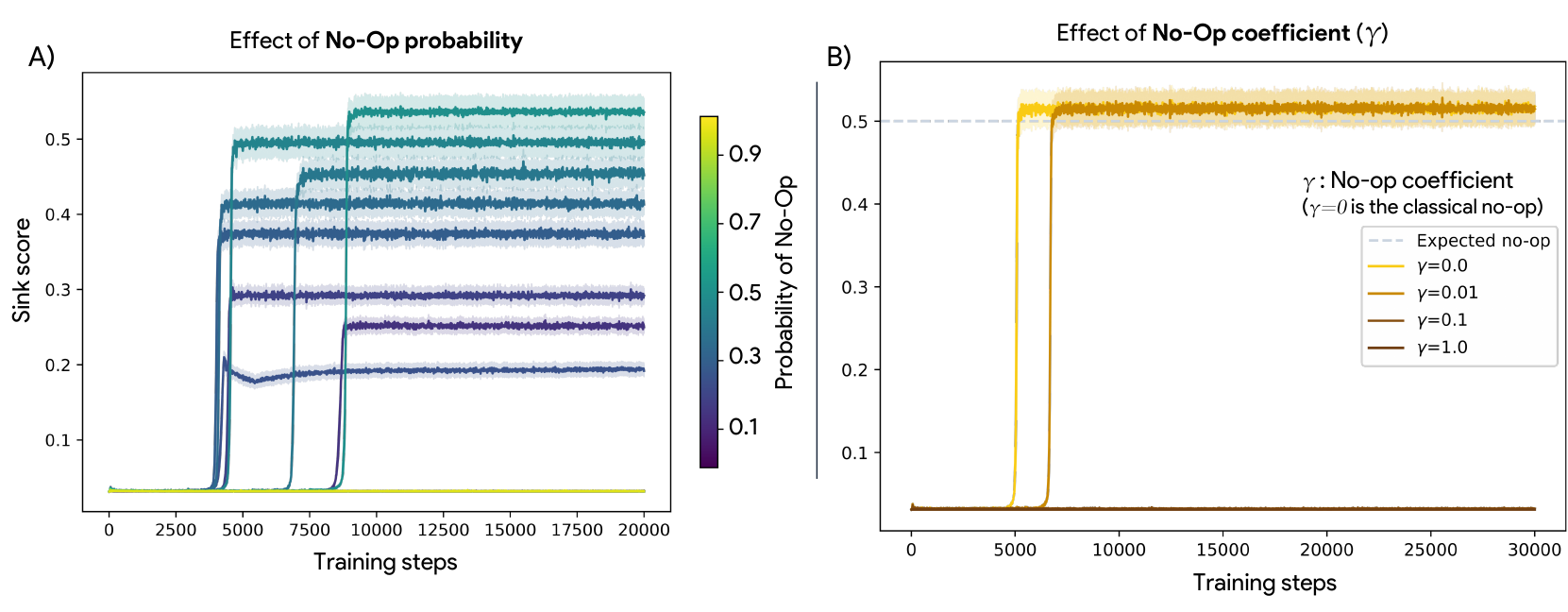}
    \caption{\textbf{Optimization sweeps for adaptive NOP.}
    \textbf{(A)} Sink emergence as a function of the probability that a token
    requires a NOP. Sink solutions emerge most rapidly when the NOP decision has
    high entropy, near $p=0.5$, but also arise when the condition is substantially
    rarer or more frequent.
    \textbf{(B)} Sink emergence as a function of the damping coefficient $\gamma$.
    Stronger suppression produces stronger sink behavior, but an exact zero target
    is not required: partially damped objectives also induce a dedicated sink.}
    \label{fig:nop-optimization-sweeps}
\end{figure}

The frequency sweep shows that the sink solution is discovered fastest when active
and inactive queries occur with comparable frequency, with the most rapid
transition near $p=0.5$. Sink formation is nevertheless robust away from this
maximum-entropy setting: when NOP queries are relatively rare, optimization still
eventually converges to a concentrated sink solution, although more slowly.

The damping sweep similarly shows that strict zero-output supervision is not
necessary for sink formation. As $\gamma$ decreases, the desired attention update
for flagged tokens approaches zero and the learned sink becomes increasingly
pronounced. Sink behavior is already present for nonzero $\gamma$, indicating that,
in these experiments, pressure to \emph{suppress} an update can favor routing through
a dedicated low-value destination even without an exact zero target.

\paragraph{Spectral signature of adaptive routing.}
We compare the query--key interaction spectrum of a model trained under the
standard NOP condition ($\lambda=0$) with a no-sink control in which the
threshold is set to $\lambda=10$, making NOP events effectively absent. Both models
are trained for 10k steps. We compute the singular values of the learned interaction
matrix
\begin{equation}
    \bm{\Theta} = \mW_Q \mW_K^\top .
\end{equation}
The sink solution exhibits a pronounced leading singular value relative to the
no-sink control (\cref{fig:no_op_signatures}B), consistent with the spectral-gating
construction in Lemma~\ref{lem:adaptivity}. In that construction, a dominant
interaction direction amplifies the feature that determines whether a query should
attend to the null token. The learned spectrum therefore provides evidence
consistent with conditional routing, complementing the near-zero sink values
shown in \cref{fig:no_op_signatures}A. A spectral spike alone does not establish
NOP behavior: the value and resulting update must also be examined.

\paragraph{Broadcast layer specialization and source values.}
The broadcast task provides observational signatures complementary to the causal
interventions in Appendix~\ref{app:broadcast-causal}. We first examine how the
two-layer model distributes computation across depth. Although the architecture
could implement partial broadcast updates in both layers, the descriptive solution
analyzed in Section~\ref{sec:broadcast} concentrates source-directed attention in
Layer 1, while Layer 2 adopts an identity-like attention pattern
(\cref{fig:broadcast_details}A). This illustrates a specialized broadcast solution;
the location of the broadcast layer is a property of the learned model, not a
requirement of the task.

For each example and layer, we compare the output-projected value norm at the
ground-truth source $j^\star$ with the distribution over non-source positions.
In the illustrated broadcast layer, source and non-source value norms are
comparable, with both distributions centered around $\|\vv\|\approx 5.5$
(\cref{fig:broadcast_details}B). Unlike adaptive-NOP sinks, the selected broadcast
values therefore remain at content scale rather than collapsing toward zero.
Together with concentrated source-directed attention, this indicates that the
source supplies a substantial update. Whether that update carries the required
payload is tested causally in Appendix~\ref{app:broadcast-causal}.

\paragraph{Query--key geometry of broadcast routing.}
We also inspect how attention identifies the content-dependent source. In the
solution shown in \cref{fig:broadcast_details}C, recipient queries become strongly
aligned, while the selected source key separates along their common direction.
In the displayed PCA coordinates, queries occupy an approximately one-dimensional
subspace on the negative side of the leading axis, and the source key extends
farther along the same direction. Their alignment supports a large positive
query--key dot product, making the source attractive to many recipients
simultaneously. The sign of the PCA axis is arbitrary; the relevant observation is
the relative alignment and separation of the queries and source key.

This geometry explains how a source selected from the input can become a sink
without occupying a fixed sequence position. It does not, by itself, establish
that $\mW_Q\mW_K^\top$ is rank one: low-dimensional projected activations do not
determine the full spectrum of the interaction matrix. We therefore treat the PCA
as a descriptive view of routing and measure the attention update's rank directly.

\paragraph{Rank structure and shared updates.}
For a nonzero attention update matrix $U \in \mathbb{R}^{L\times d}$, where $L$
is the number of recipient tokens, we use the stable rank
\cite{neyshabur2015norm,arora2018stronger},
\begin{equation}
    \operatorname{sr}(U)
    =
    \frac{\|U\|_F^2}{\|U\|_2^2}.
\end{equation}
Here $\|\cdot\|_F$ and $\|\cdot\|_2$ denote the Frobenius and spectral norms,
respectively. The broadcast layer has stable rank approximately one, whereas
the second layer produces higher-rank updates
(\cref{fig:broadcast_details}D). This agrees with Lemma~\ref{cor:rank1}: when one
source dominates attention, recipient updates are approximately scalar multiples
of the same value vector.

The low-rank property concerns the update, not necessarily the residual-stream
representations. In particular, adding an identical vector to every token preserves
centered variance, as formalized in Lemma~\ref{cor:similarity}. We therefore use
approximately rank-one updates together with content-scale source values and
source-directed attention as evidence of broadcast-like computation, rather than
interpreting low rank alone as representation collapse.

For completeness, we also monitor the magnitude of the pre-output attention update
$(\mA\mV)_i$ and the cosine similarity of query vectors across recipient positions.
Small variation in update magnitude together with highly aligned queries provides
an additional indication that the selected value is distributed broadly. These
quantities remain descriptive diagnostics; the cross-example interventions in
Appendix~\ref{app:broadcast-causal} provide the stronger causal test that the
selected value carries the broadcast payload.

\section{Formal statements and proofs}
\label{app:hyp}

We start by recalling our definition of a sink:
\begin{definition}[$\epsilon$-Sink]
A position $s$ is an $\epsilon$-sink if $\A_{is} \geq 1 - \epsilon$ for all positions $i \in \{1, \ldots, n\}$, where $\epsilon \in [0, 1)$ is small.
\end{definition}

\subsection{Proof of Lemma~\ref{lem:noop_unique}: Uniqueness of the Exact No-Op Solution}
\label{sec:noop_unique}

\begin{proof}[Proof.]
For any query $i \in I$, write
\begin{equation}
    \vo_i
    = \A_{is}\vv_s + \sum_{j \neq s}\A_{ij}\vv_j
    = (1-\epsilon_i)\vv_s + \vr_i,
\end{equation}
where $\epsilon_i = 1-\A_{is} \in [0,\epsilon]$.
Since attention weights are nonnegative and sum to one,
the residual satisfies
$\|\vr_i\| \leq \epsilon_i V_{\max} \leq \epsilon V_{\max}$.
The triangle inequality therefore gives
\begin{equation}
    (1-\epsilon)\|\vv_s\|
    \leq (1-\epsilon_i)\|\vv_s\|
    = \|\vo_i-\vr_i\|
    \leq \delta+\epsilon V_{\max},
\end{equation}
which proves the stated bound.

At finite $\epsilon$, setting $\vv_s=\bm{0}$ gives
$\|\vo_i\| \leq \epsilon V_{\max}$.
Alternatively, exact cancellation requires
$\vv_s=-\vr_i/(1-\epsilon_i)$ consistently across $i \in I$,
but such a value remains subject to the same bound.
Thus cancellation cannot sustain a nonzero sink value as
$\epsilon,\delta \to 0$ with uniformly bounded $V_{\max}$.

For $\epsilon=0$, all residual contributions vanish and
$\vo_i=\vv_s$ for every $i \in I$.
Consequently, exact no-op holds if and only if
$\vv_s=\bm{0}$, whereas a positive tolerance $\delta$
permits any sink value with norm at most $\delta$.
\end{proof}

\subsection{Mechanisms for Adaptivity: Statement and Proof}
\label{app:adaptivity_proof}

Conditional suppression requires both a near-zero sink value and a routing
mechanism that selects it for the appropriate queries. We examine how a linear
query--key interaction can implement this conditional routing. The relevant
requirement is that the sink logit exceed competing logits: softmax attention
approaches, rather than attains, a perfect gate at finite parameter values.

\begin{lemma}[\textbf{Mechanisms for Adaptivity}]
\label{lem:adaptivity}
Consider queries $\vq_i=\mW_Q^\top\vx_i$, a sink key
$\vk_s=\mW_K^\top\vx_s$, and interaction matrix
$\bm{\Theta}=\mW_Q\mW_K^\top$. The sink logit satisfies
\begin{equation}
    S_{is}
    =
    \frac{\vx_i^\top\bm{\Theta}\vx_s}{\sqrt{d_h}},
    \qquad
    |S_{is}|
    \leq
    \frac{\|\vx_i\|\,\|\bm{\Theta}\|_2\,\|\vx_s\|}
         {\sqrt{d_h}}.
\end{equation}
For selected queries, $\A_{is}$ approaches one when $S_{is}$ exceeds
competing logits by a sufficiently large margin. In constructions aligned
with a triggering feature, this margin can be increased through either
\textbf{(i) spectral gating}, by increasing a leading singular value of
$\bm{\Theta}$ along that feature direction, or
\textbf{(ii) massive activation}, by increasing the sink activation norm
along the corresponding direction, provided competing logits remain
controlled.
\end{lemma}

\begin{proof}
Consider a simplified attention mechanism with a single head. Let $\vx_i$ be a query token and $\vx_s$ be the dedicated sink token. The attention score (logit) for the sink is given by:
\begin{equation}
    S_{is} = \frac{1}{\sqrt{d_h}} \vx_i^\top \bm{W}_Q \bm{W}_K^\top \vx_s = \frac{1}{\sqrt{d_h}} \vx_i^\top \bm{\Theta} \vx_s.
\end{equation}
We assume the gating decision is based on a specific feature direction $\vu$ (with $\|\vu\|=1$). 
We decompose the query as $\vx_i = \lambda_i \vu + \vx_{\perp}$, where $\lambda_i$ is the scalar activation of the feature triggering the \nop, and $\vx_{\perp} \perp \vu$.

To consistently route attention to the sink when the feature is active ($\lambda_i > 0$), the model aligns the interaction matrix $\bm{\Theta}$ and the sink token $\vx_s$ with this direction. We construct:
\begin{itemize}
    \item The interaction matrix as rank-1: $\bm{\Theta} = \sigma \vu \vu^\top$, where $\sigma = \|\bm{\Theta}\|_2$ is the spectral norm.
    \item The sink token as aligned with the flag: $\vx_s = \gamma_s \vu$, where $\gamma_s = \|\vx_s\|$ is the token norm.
\end{itemize}

Substituting these into the attention score equation:
\begin{equation}
    S_{is} = \frac{1}{\sqrt{d_h}} (\lambda_i \vu + \vx_{\perp})^\top (\sigma \vu \vu^\top) (\gamma_s \vu) = \frac{\lambda_i \sigma \gamma_s}{\sqrt{d_h}}.
\end{equation}

For the attention weight $\A_{is}$ to approach 1 (hard gating), the score $S_{is}$ must exceed the scores of all other $N$ tokens by a large margin $M$ (where $M$ depends on the desired sharpness of the softmax). That is:
\begin{equation}
    \frac{\lambda_i \sigma \gamma_s}{\sqrt{d_h}} \gg \max_{j \neq s} S_{ij} \implies \sigma \gamma_s \gg \frac{M \sqrt{d_h}}{\lambda_i}.
\end{equation}

This inequality $\sigma \gamma_s \gg C$ reveals the two distinct optimization paths available to the model to satisfy the gating condition: \textbf{(i) Spectral Gating:} The model increases $\sigma$ (the singular value of $\bm{W}_Q \bm{W}_K^\top$) while keeping the sink norm $\gamma_s$ moderate. \textbf{(ii) Massive Activation:} The model increases $\gamma_s$ (the norm of the sink token $\vx_s$) while keeping the weight spectral norm $\sigma$ moderate.
Both solutions are mathematically equivalent for producing the necessary logit $S_{is}$ to trigger the sink.
\end{proof}

The construction separates two ways to strengthen the same conditional
routing operation: increasing the query--key interaction along a relevant
direction or increasing the sink activation along that direction.
This offers a possible interpretation of the massive activations observed
in outlier tokens~\cite{darcet2023vision}: they can support reliable routing
to a null value without requiring a comparably large interaction-matrix
norm. These are sufficient constructions, not an exhaustive classification
of learned routing mechanisms. Moreover, neither a spectral spike nor a
large activation establishes NOP behavior on its own; conditional routing
must be paired with a negligible value contribution.

\subsection{Proof of Lemma~\ref{cor:rank1}: Rank-1 Update}

\begin{proof}[Proof.]
The attention output matrix is given by $\mO = \A \mV$, where $\A \in \mathbb{R}^{n \times n}$ is the row-stochastic attention matrix and $\mV \in \mathbb{R}^{n \times d_h}$ is the matrix of value vectors with rows $\vv_1^\top, \ldots, \vv_n^\top$.

Let $s$ be an $\epsilon$-broadcast sink. We decompose the attention matrix as:
\begin{equation}
    \A = \A_s \bm{e}_s^\top + \mE,
\end{equation}
where $\bm{a}_s = (\A_{1s}, \ldots, \A_{ns})^\top \in \mathbb{R}^n$ is the $s$-th column of $\A$, $\ve_s$ is the $s$-th standard basis vector, and $\mE$ is the residual matrix with entries $\mE_{ij} = \A_{ij}$ for $j \neq s$ and $\mE_{is} = 0$.

Since $\A$ is row-stochastic and $\A_{is} \geq 1 - \epsilon$, we have $\sum_{j \neq s} \A_{ij} \leq \epsilon$ for each row $i$. Thus, each row of $\mE$ has $\ell_1$-norm at most $\epsilon$, yielding $\|\mE\|_{\infty} \leq \epsilon$.

Multiplying by $\mV$:
\begin{equation}
    \mO = \A \mV = \bm{a}_s \ve_s^\top \mV + \mE \mV = \bm{a}_s \vv_s^\top + \mE \mV.
\end{equation}
The first term $\bm{a}_s \vv_s^\top$ is rank-1. The residual satisfies:
\begin{equation}
    \|\mE \mV\|_{\infty} \leq \|\mE\|_{\infty} \max_j \|\vv_j\| \leq \epsilon \cdot \max_j \|\vv_j\|.
\end{equation}
Therefore:
\begin{equation}
    \left\| \mO - \bm{a}_s \vv_s^\top \right\|_{\infty} \leq \epsilon \cdot \max_j \|\vv_j\|.
\end{equation}

The rank-1 matrix $\bm{a}_s \vv_s^\top$ has singular value decomposition $\sigma_1 \vu_1 \vw_1^\top$, where $\sigma_1 = \|\bm{a}_s\| \|\vv_s\|$, $\vu_1 = \bm{a}_s / \|\bm{a}_s\|$, and $\vw_1 = \vv_s / \|\vv_s\|$. Since $\A_{is} \geq 1 - \epsilon$ for all $i$, we have $\|\bm{a}_s\| \geq \sqrt{n}(1 - \epsilon)$. Thus, the dominant singular value satisfies:
\begin{equation}
    \sigma_1 \geq \sqrt{n}(1 - \epsilon) \|\vv_s\|,
\end{equation}
and the right singular vector aligns exactly with $\vv_s / \|\vv_s\|$.
\end{proof}

\subsection{Shared Additive Component: Statement and Proof}

A broadcast sink introduces a shared component into recipient updates,
but this does not imply that their residual-stream representations
collapse. The following result bounds the change in centered token
variance and distinguishes a shared additive update from contraction
toward the token mean.

\begin{lemma}[\textbf{Shared Additive Component}]
\label{cor:similarity}
Suppose $s$ is an $\epsilon$-sink for all queries in the attention update
\begin{equation}
    \vx_i^{\ell+1}
    =
    \vx_i^\ell+\sum_j\mA_{ij}\vv_j.
\end{equation}
Let
\begin{equation}
    \bar{\vx}^{\ell}
    =
    \frac{1}{n}\sum_i\vx_i^\ell,
    \qquad
    \mathcal{V}_{\ell}
    =
    \frac{1}{n}\sum_i
    \|\vx_i^\ell-\bar{\vx}^{\ell}\|^2,
\end{equation}
and define the variance of the sink attention weights by
\begin{equation}
    \bar{A}_s
    =
    \frac{1}{n}\sum_i\mA_{is},
    \qquad
    \mathbb{V}(\mA_s)
    =
    \frac{1}{n}\sum_i(\mA_{is}-\bar{A}_s)^2.
\end{equation}
With value norms and $\mathcal{V}_{\ell}$ bounded independently
of $\epsilon$, we have
\begin{equation}
    \mathcal{V}_{\ell+1}
    \leq {}
    \mathcal{V}_{\ell}
    +2\|\vv_s\|
      \sqrt{\mathcal{V}_{\ell}\,\mathbb{V}(\mA_s)}
    +\|\vv_s\|^2\mathbb{V}(\mA_s)
    +\mathcal{O}(\epsilon).
\end{equation}
\end{lemma}

\begin{proof}[Proof.]
Let $\vx_i^{\ell+1} = \vx_i^{\ell} + \vo_i$ where
$\vo_i = \sum_j \A_{ij} \vv_j$. Let $s$ be an
$\epsilon$-broadcast sink, so that $\A_{is} \geq 1-\epsilon$
for every query $i$.

We first bound the deviation from the broadcast approximation.
For each position $i$:
\begin{align}
    \vo_i &= \A_{is} \vv_s + \sum_{j \neq s} \A_{ij} \vv_j.
\end{align}
The residual term satisfies:
\begin{equation}
    \left\| \sum_{j \neq s} \A_{ij} \vv_j \right\|
    \leq \sum_{j \neq s} \A_{ij} \|\vv_j\|
    \leq \epsilon \cdot \max_j \|\vv_j\|.
\end{equation}
Thus, $\vo_i = \A_{is} \vv_s + \vr_i$ where
$\|\vr_i\| \leq \epsilon \cdot \max_j \|\vv_j\|$.

The mean representation after the block is:
\begin{equation}
    \bar{\vx}^{\ell+1} = \bar{\vx}^{\ell} + \bar{\vo},
\end{equation}
where $\bar{\vo} = \frac{1}{n} \sum_i \vo_i
= \bar{a}_s \vv_s + \bar{\vr}$, with
$\bar{a}_s = \frac{1}{n} \sum_i \A_{is}$ and
$\|\bar{\vr}\| \leq \epsilon \cdot \max_j \|\vv_j\|$.

The centered representation is:
\begin{equation}
    \vx_i^{\ell+1} - \bar{\vx}^{\ell+1}
    = \left( \vx_i^{\ell} - \bar{\vx}^{\ell} \right)
    + (\A_{is} - \bar{a}_s) \vv_s
    + (\vr_i - \bar{\vr}).
\end{equation}

Let $\vd_i^{\ell} = \vx_i^{\ell} - \bar{\vx}^{\ell}$ and
$\delta_i = \A_{is} - \bar{a}_s$. The variance after the block is:
\begin{align}
    \frac{1}{n} \sum_i
    \left\| \vx_i^{\ell+1} - \bar{\vx}^{\ell+1} \right\|^2
    &= \frac{1}{n} \sum_i
    \left\| \vd_i^{\ell} + \delta_i \vv_s
    + (\vr_i - \bar{\vr}) \right\|^2.
\end{align}

Expanding gives:
\begin{align}
    &= \frac{1}{n} \sum_i \left\| \vd_i^{\ell} \right\|^2
    + \frac{1}{n} \sum_i \delta_i^2 \|\vv_s\|^2
    + \frac{1}{n} \sum_i \|\vr_i - \bar{\vr}\|^2 \nonumber \\
    &\quad + \frac{2}{n} \sum_i \delta_i
    \left\langle \vd_i^{\ell}, \vv_s \right\rangle
    + \frac{2}{n} \sum_i
    \left\langle \vd_i^{\ell}, \vr_i - \bar{\vr} \right\rangle
    \nonumber \\
    &\quad + \frac{2}{n} \sum_i \delta_i
    \left\langle \vv_s, \vr_i - \bar{\vr} \right\rangle.
\end{align}

The second term equals $\mathbb{V}(\bm{a}_s) \|\vv_s\|^2$,
where $\mathbb{V}(\bm{a}_s) = \frac{1}{n} \sum_i \delta_i^2$
is the variance of attention weights to the sink.
Since $\|\vr_i-\bar{\vr}\| \leq 2\epsilon \max_j\|\vv_j\|$
and $|\delta_i| \leq \epsilon$, the terms involving $\vr_i$
have total absolute value $\mathcal{O}(\epsilon)$, provided
value norms and the input token variance
$\frac{1}{n}\sum_i\|\vd_i^\ell\|^2$
are bounded independently of $\epsilon$.

For the cross-term
$\frac{2}{n} \sum_i \delta_i
\langle \vd_i^{\ell}, \vv_s \rangle$,
by Cauchy--Schwarz:
\begin{equation}
    \left| \frac{2}{n} \sum_i \delta_i
    \left\langle \vd_i^{\ell}, \vv_s \right\rangle \right|
    \leq 2 \sqrt{\mathbb{V}(\bm{a}_s)}
    \cdot \sqrt{\frac{1}{n} \sum_i
    \left\| \vd_i^{\ell} \right\|^2}
    \cdot \|\vv_s\|.
\end{equation}

Combining these bounds:
\begin{align}
    \frac{1}{n} \sum_i
    \left\| \vx_i^{\ell+1} - \bar{\vx}^{\ell+1} \right\|^2
    &\leq \frac{1}{n} \sum_i
    \left\| \vx_i^{\ell} - \bar{\vx}^{\ell} \right\|^2
    \nonumber \\
    &\quad + 2\|\vv_s\|\sqrt{\mathbb{V}(\bm{a}_s)}
    \sqrt{\frac{1}{n} \sum_i
    \left\| \vx_i^{\ell} - \bar{\vx}^{\ell} \right\|^2}
    \nonumber \\
    &\quad + \|\vv_s\|^2 \mathbb{V}(\bm{a}_s)
    + \mathcal{O}(\epsilon).
\end{align}

When attention to the sink is uniform, i.e., $\A_{is} = c$
for all $i$ with $c \geq 1-\epsilon$, we have
$\mathbb{V}(\bm{a}_s) = 0$ and $\delta_i = 0$ for all $i$.
In this case, the variance expansion gives:
\begin{equation}
    \frac{1}{n} \sum_i
    \left\| \vx_i^{\ell+1} - \bar{\vx}^{\ell+1} \right\|^2
    = \frac{1}{n} \sum_i
    \left\| \vx_i^{\ell} - \bar{\vx}^{\ell} \right\|^2
    + \mathcal{O}(\epsilon).
\end{equation}
The variance is approximately preserved, while all tokens
receive the same additive component $c \cdot \vv_s$,
shifting collectively along the direction $\vv_s$.
For a perfect sink ($\epsilon=0$), the residual terms vanish
and centered variance is preserved exactly.
\end{proof}

For a perfect sink, $\mA_{is}=1$ for every recipient, so each token
receives the same additive vector $\vv_s$. Subtracting the updated
token mean removes this shared component exactly:
\[
    \vx_i^{\ell+1}-\bar{\vx}^{\ell+1}
    =
    \vx_i^\ell-\bar{\vx}^{\ell}.
\]
Consequently, $\mathcal{V}_{\ell+1}=\mathcal{V}_{\ell}$.
For approximate sinks, variation in recipient weights and residual
contributions from other values can change centered variance, as
captured by the bound. Thus, an approximately rank-one attention
update is a signature of shared communication, not by itself evidence
of representation collapse or a guaranteed increase in pairwise
cosine similarity.

\section{LeJEPA Mitigation Experiments: Training Details}
\label{app:mitigation_exp}

We train two sets of LeJEPA ViT-L models from scratch on ImageNet-1K. The first is a sweep over registers and attention gating with three seeds per variant (Appendix~\ref{app:lejepa_sweep}; Tables~\ref{tab:lejepa_imagenet} and~\ref{tab:lejepa_dense_prediction_app}). The second is the four single-seed models of Table~\ref{tab:lejepa_full_results} in Section~\ref{sec:mitigation} (Appendix~\ref{app:lejepa_table1}). Both sets share the recipe of Tables~\ref{tab:lejepa_architecture}--\ref{tab:lejepa_augmentation}. Appendix~\ref{app:lejepa_table1} lists where the models of Table~\ref{tab:lejepa_full_results} differ from it, and Appendix~\ref{app:lejepa_probes} gives the probe protocols of both sets.

\subsection{Common recipe and the register and gating sweep}
\label{app:lejepa_sweep}

We trained ViT-L models from scratch on ImageNet-1K for 100 epochs using the LeJEPA self-supervised objective~\cite{balestriero2025lejepa}, which combines a SIGReg regularization loss and a multi-view invariance loss. Architecture specifications are provided in Table~\ref{tab:lejepa_architecture}, optimization hyperparameters in Table~\ref{tab:lejepa_optimization}, regularization in Table~\ref{tab:lejepa_regularization}, and data augmentation in Table~\ref{tab:lejepa_augmentation}. Each model was trained on 4 NVIDIA H100 GPUs using PyTorch DDP.

We evaluate six variants: (1)~Baseline, (2)~Registers (4 tokens), (3)~Elementwise gating, (4)~Elementwise gating + Registers, (5)~Headwise gating, and (6)~Headwise gating + Registers, each trained with three random seeds (9000, 9001, 9002). These models provide the register comparisons: per-seed ImageNet-1K linear probe results are provided in Table~\ref{tab:lejepa_imagenet} and dense-prediction probes in Table~\ref{tab:lejepa_dense_prediction_app}. They are not the models of Table~\ref{tab:lejepa_full_results}, which are described in Appendix~\ref{app:lejepa_table1}.
% Superseded (Table 1 no longer uses the sweep models):
% For results reported in Table~\ref{tab:lejepa_full_results}, we select the better-performing gating variant (elementwise vs.\ headwise) based on mean ImageNet-1K linear probe accuracy across seeds.

\begin{table}[h]
\small
\centering
\caption{ViT-Large architecture specifications for LeJEPA experiments.}
\label{tab:lejepa_architecture}
\begin{tabular}{ll}
\toprule
\textbf{Parameter} & \textbf{Value} \\
\midrule
Image size & $224 \times 224$ \\
Patch size & $16 \times 16$ \\
Number of patches & 196 \\
Embedding dimension & 1024 \\
Transformer blocks & 24 \\
Attention heads & 16 \\
Head dimension & 64 \\
MLP ratio & 4.0 \\
MLP hidden dimension & 4096 \\
QKV bias & False \\
Register tokens (when enabled) & 4 \\
LayerScale init (attention and MLP branches) & $10^{-4}$ (see Table~\ref{tab:lejepa_table1_diff}) \\
Readout & CLS token \\
Projector & $1024 \to 2048 \to 2048 \to 128$ (BN + ReLU) \\
Projector output dimension & 128 \\
\bottomrule
\end{tabular}
\end{table}

\begin{table}[h]
\small
\centering
\caption{Optimization hyperparameters for LeJEPA experiments.}
\label{tab:lejepa_optimization}
\begin{tabular}{ll}
\toprule
\textbf{Parameter} & \textbf{Value} \\
\midrule
Optimizer & AdamW \\
Learning rate & $5 \times 10^{-4}$ \\
Weight decay & 0.05 \\
Warmup epochs & 10 \\
Total epochs & 100 \\
LR schedule & Linear warmup + cosine annealing to $\text{lr}/1000$ \\
Gradient clip & 1.0 \\
Batch size (per GPU) & 64 \\
Effective batch size & 256 (4 GPUs) \\
Precision & Mixed (bfloat16) \\
Online probe LR & $1 \times 10^{-3}$ \\
Number of views & 4 \\
SIGReg weight $\lambda$ & 0.02 \\
\bottomrule
\end{tabular}
\end{table}

\begin{table}[h]
\small
\centering
\caption{Regularization hyperparameters for LeJEPA experiments.}
\label{tab:lejepa_regularization}
\begin{tabular}{ll}
\toprule
\textbf{Technique} & \textbf{Value} \\
\midrule
Drop path (stochastic depth) & 0.1 (increasing linearly over depth) \\
Dropout & 0.0 \\
Attention dropout & 0.0 \\
\bottomrule
\end{tabular}
\end{table}

\begin{table}[h]
\small
\centering
\caption{Training data augmentation pipeline for LeJEPA experiments. Four augmented views are generated per image.}
\label{tab:lejepa_augmentation}
\begin{tabular}{ll}
\toprule
\textbf{Augmentation} & \textbf{Parameters} \\
\midrule
Random resized crop & Scale: $(0.08, 1.0)$ \\
Color jitter & Brightness/contrast/saturation $= 0.8$, hue $= 0.2$, $p = 0.8$ \\
Random grayscale & $p = 0.2$ \\
Gaussian blur & Kernel size $= 7$, $\sigma \in (0.1, 2.0)$, $p = 0.5$ \\
Random solarize & Threshold $= 128$, $p = 0.2$ \\
Random horizontal flip & $p = 0.5$ \\
Normalization & ImageNet statistics \\
\midrule
\multicolumn{2}{l}{\textit{Validation only}} \\
\midrule
Resize & $256 \times 256$ \\
Center crop & $224 \times 224$ \\
\bottomrule
\end{tabular}
\end{table}

\begin{table}[h]
\small
\centering
\caption{ImageNet-1K linear probe top-1 accuracy (\%) for each LeJEPA variant across three seeds.}
\label{tab:lejepa_imagenet}
\begin{tabular}{lccc}
\toprule
\textbf{Variant} & \textbf{9000} & \textbf{9001} & \textbf{9002} \\
\midrule
Baseline                  & 69.31 & 69.37 & 69.44 \\
Registers                 & 69.14 & 69.37 & 69.42 \\
Elementwise gating        & 69.56 & 69.49 & 69.52 \\
Elementwise gating + Reg. & 69.28 & 69.46 & 69.35 \\
Headwise gating           & 69.02 & 69.05 & 69.13 \\
Headwise gating + Reg.    & 69.28 & 69.15 & 68.95 \\
\bottomrule
\end{tabular}
\end{table}

\begin{table}[t] \centering \caption{Linear probing performance on dense prediction tasks for the register and gating sweep. We report mean and standard deviation over three independently trained LeJEPA backbones. Probe protocol in Tables~\ref{tab:lejepa_probe_settings} and~\ref{tab:lejepa_probe_diff} (augmented; 10 epochs on ADE20K, 25 on VOC12). These numbers use a different protocol and different backbones from Table~\ref{tab:lejepa_full_results} and are not comparable with it.} \label{tab:lejepa_dense_prediction_app} \begin{tabular}{ l r@{${}\pm{}$}l r@{${}\pm{}$}l } \toprule Variant & \multicolumn{2}{c}{ADE20K} & \multicolumn{2}{c}{VOC12} \\ \midrule Baseline & $0.166$ & $0.006$ & $0.438$ & $0.016$ \\ Registers & $0.187$ & $0.002$ & $0.499$ & $0.001$ \\ Headwise gating & $0.185$ & $0.0003$ & $0.476$ & $0.006$ \\ Elementwise gating & $0.186$ & $0.004$ & $0.480$ & $0.002$ \\ Headwise + registers & $0.199$ & $0.003$ & $\mathbf{0.524}$ & $\mathbf{0.013}$ \\ Elementwise + registers & $\mathbf{0.200}$ & $\mathbf{0.006}$ & $0.522$ & $0.017$ \\ \bottomrule \end{tabular} \end{table}

\subsection{The models of Table~\ref{tab:lejepa_full_results}}
\label{app:lejepa_table1}

Table~\ref{tab:lejepa_full_results} compares four models: Baseline, Gating, Global + Penalty, and Gating + Global + Penalty. Each is a single backbone (seed 9000), trained without registers using the recipe of Tables~\ref{tab:lejepa_architecture}--\ref{tab:lejepa_augmentation} on 4 NVIDIA H100 GPUs (per-GPU batch 64, effective batch 256). Baseline and Gating are separate runs from the seed-9000 models of the sweep. All four models are trained for 50 epochs with the learning-rate schedule of Table~\ref{tab:lejepa_optimization}, which is defined over 100 epochs (10 warm-up epochs, then cosine decay to $\text{lr}/1000$ at epoch 100). At epoch 50 the learning rate is therefore $2.9\times10^{-4}$, not yet fully annealed. Every number in Tables~\ref{tab:lejepa_full_results}, \ref{tab:lejepa_sink_counts} and~\ref{tab:lejepa_eps_sensitivity} is measured on the checkpoint saved after epoch 50. Table~\ref{tab:lejepa_table1_diff} lists the settings that differ between the four models.

\begin{table}[h]
\small
\centering
\caption{Settings that differ between the four models of Table~\ref{tab:lejepa_full_results}. Everything else follows Tables~\ref{tab:lejepa_architecture}--\ref{tab:lejepa_augmentation}. Gate: elementwise G1 attention output gate. Pathway: Eq.~\ref{eq:global_pathway} in all 24 blocks. Penalty: Eq.~\ref{eq:broadcast_penalty}; $^{a}$computed on the attention update without the gate. Time per epoch: median over the first 50 epochs on 4 H100 GPUs.}
\label{tab:lejepa_table1_diff}
\setlength{\tabcolsep}{5pt}
\begin{tabular}{lcccc}
\toprule
 & \textbf{Baseline} & \textbf{Gating} & \makecell{\textbf{Global +}\\\textbf{Penalty}} & \makecell{\textbf{Gating + Global}\\\textbf{+ Penalty}} \\
\midrule
Attention gate & -- & G1 & -- & G1 \\
LayerScale init (attn., MLP) & $10^{-4}$ & $10^{-4}$ & $1.0$ & $1.0$ \\
Global pathway & -- & -- & yes & yes \\
Penalty $\lambda_B$ & -- & -- & $0.1$ & $0.1^{a}$ \\
Added parameters & -- & 25.2M & 3.19M & 28.4M \\
Time per epoch & 33 min & 37 min & 43 min & 51 min \\
\bottomrule
\end{tabular}
\end{table}

\paragraph{Gating.}
Following the G1 variant of \citet{qiu2025gated}, each attention block multiplies the concatenated head outputs elementwise by $\sigma(\mW_\theta\widetilde{\vx}_i+\bm{b}_\theta)\in(0,1)^{d}$ before the output projection. The gate is computed from the block's normalized input $\widetilde{\vx}_i$ with a full $d\times d$ map.

\iffalse
\paragraph{Global pathway.}
The pathway of Eq.~\ref{eq:global_pathway} is added to every block alongside attention and reads the same normalized tokens $\widetilde{\vx}^\ell$, including \texttt{[CLS]}. We use a single global state with $d_g=64$ ($\bm{w}_P^\ell\in\mathbb{R}^{d}$, $\mW_{V,G}^\ell\in\mathbb{R}^{64\times d}$, $\mW_{O,G}^\ell\in\mathbb{R}^{d\times 64}$). No bias is used except in the recipient gate, and no LayerScale is applied to the pathway output. All pathway weights are initialized from a truncated normal with standard deviation 0.02 and gate biases at zero, so $\gamma_i^\ell=0.5$ at initialization. In Gating + Global + Penalty the recipient gate is hierarchical:
\begin{equation}
\label{eq:hier_gate}
\big(a_i^\ell,\,r_{i,0}^\ell,\,r_{i,\mathrm{bc}}^\ell\big)=\mW_\gamma^\ell\widetilde{\vx}_i^\ell+\bm{b}_\gamma^\ell,
\qquad
\gamma_i^\ell=\sigma\!\left(a_i^\ell\right)\frac{\exp\!\left(r_{i,\mathrm{bc}}^\ell\right)}{\exp\!\left(r_{i,0}^\ell\right)+\exp\!\left(r_{i,\mathrm{bc}}^\ell\right)},
\end{equation}
with $\mW_\gamma^\ell\in\mathbb{R}^{3\times d}$. The first factor decides whether token $i$ takes an update, and the second chooses between the broadcast message and no update. At initialization $\gamma_i^\ell=0.25$. The pathway update stays rank one across tokens.
\fi

\paragraph{Broadcast penalty.}
Both penalized models add the penalty of Eq.~\ref{eq:broadcast_penalty} with $\lambda_B=0.1$, $\mathcal{L}_B=\{1,\dots,23\}$ (block 0 unpenalized) and $\epsilon=10^{-12}$. It is active from the first training step with its full weight, with no warm-up or ramp. The ratio is computed per block on each GPU's mini-batch (all views), and the per-block ratios are averaged over $\mathcal{L}_B$. The common component $\bar{\vo}^{\ell,\mathrm{attn}}$ averages over all tokens, including \texttt{[CLS]}. Gradients flow through both numerator and denominator. In Gating + Global + Penalty, $\mO^{\ell,\mathrm{attn}}$ is the output-projected attention update after LayerScale, computed without the elementwise gate. The global-pathway update is not penalized.

\subsection{Linear probes}
\label{app:lejepa_probes}

\paragraph{ImageNet-1K.}
The ImageNet column of Table~\ref{tab:lejepa_full_results} is the LeJEPA trainer's online linear probe. It is a LayerNorm followed by a linear layer on the stop-gradient \texttt{[CLS]} embedding, trained jointly with the backbone on every training view with cross-entropy (AdamW, learning rate $10^{-3}$, weight decay $10^{-7}$, same warm-up and cosine schedule as the backbone). We report top-1 accuracy on the 50,000 validation images (resize 256, center crop 224) after epoch 50. Each variant has a single backbone and this probe has no seed of its own, so the column has no error bar.

\paragraph{Semantic segmentation.}
Table~\ref{tab:lejepa_probe_settings} lists the ADE20K and Pascal VOC 2012 probe settings shared by Table~\ref{tab:lejepa_full_results} and the sweep (Table~\ref{tab:lejepa_dense_prediction_app}). Table~\ref{tab:lejepa_probe_diff} lists where they differ: augmentation, probe length and what the $\pm$ measures. For Table~\ref{tab:lejepa_full_results}, the frozen features are extracted once per backbone and dataset at a fixed resize and cached, so the probe sees no augmentation. The VOC12 probe runs 100 epochs because at 25 epochs it has not converged: Baseline reaches 0.49, 0.55 and 0.60 mIoU after 25, 50 and 100 epochs. The ordering of the four models is the same at all three lengths.

\begin{table}[h]
\small
\centering
\caption{Linear segmentation probe settings shared by Table~\ref{tab:lejepa_full_results} and the sweep (Table~\ref{tab:lejepa_dense_prediction_app}). The backbone is frozen and only the linear head is trained.}
\label{tab:lejepa_probe_settings}
\begin{tabular}{ll}
\toprule
\textbf{Setting} & \textbf{Value} \\
\midrule
Backbone & frozen (eval mode, fp32); final-norm patch tokens only \\
Input resolution & $512\times512$; positional embeddings interpolated bicubically \\
Head & linear $1024\to C$ per patch; logits upsampled bilinearly to $512\times512$ \\
Loss & cross-entropy; background and unlabeled pixels ignored \\
Classes $C$ & 150 (ADE20K), 20 (VOC12) \\
Optimizer & AdamW, lr $5\times10^{-3}$, weight decay $10^{-2}$, batch 32, gradient clip 1.0 \\
LR schedule & 100-iteration linear warm-up, then cosine decay to $1\%$ of the lr \\
Data splits (train / val) & ADE20K 20,210 / 2,000; VOC12 1,464 / 1,449 \\
Reported value & validation mIoU at the best probe epoch \\
\bottomrule
\end{tabular}
\end{table}

\begin{table}[h]
\small
\centering
\caption{Linear segmentation probe settings that differ between Table~\ref{tab:lejepa_full_results} and the sweep (Table~\ref{tab:lejepa_dense_prediction_app}). Table~\ref{tab:lejepa_full_results} probes are trained on features extracted once per backbone at a fixed $512\times512$ resize and stored in fp16, so they see no augmentation. Sweep probes use a random resized crop (scale 0.5--1, aspect ratio 3/4--4/3) and a horizontal flip.}
\label{tab:lejepa_probe_diff}
\begin{tabular}{lcc}
\toprule
\textbf{Setting} & \textbf{Table~\ref{tab:lejepa_full_results}} & \textbf{ Table~\ref{tab:lejepa_dense_prediction_app}} \\
\midrule
Train augmentation & none & crop + flip \\
Probe epochs, ADE20K & 25 & 10 \\
Probe epochs, VOC12 & 100 & 25 \\
Backbone seeds & 9000 & 9000, 9001, 9002 \\
Backbone checkpoint & epoch 50 & epoch 100 \\
Probe seeds, ADE20K & 42, 43, 44 & 42 \\
Probe seeds, VOC12 & 46, 47, 48 & 46 \\
$\pm$ & std over probe seeds & std over backbones \\
\bottomrule
\end{tabular}
\end{table}

\paragraph{Compute resources.}
Experiments were run on a compute cluster using NVIDIA H100 GPUs. The LeJEPA ViT-L experiments were trained with PyTorch DDP on 4 H100 GPUs per run. For the sweep, each 100-epoch run took approximately 3 days, corresponding to roughly 300 H100 GPU-hours per seed, and we trained three seeds for each variant. The four models of Table~\ref{tab:lejepa_full_results} took 33 to 51 minutes per epoch (Table~\ref{tab:lejepa_table1_diff}), so the 50 epochs behind Table~\ref{tab:lejepa_full_results} correspond to roughly 110--170 H100 GPU-hours per model. Training linear semantic-segmentation probes on frozen LeJEPA backbones required one GPU per probe and only a few GPU-hours per probe. For Table~\ref{tab:lejepa_full_results}, features are extracted once per backbone and dataset, after which each probe takes under an hour. Similarly, toy-model experiments and pretrained-model diagnostics were run on a single GPU and were substantially cheaper than the LeJEPA training runs, typically completing within a few GPU-hours per run. All experiments were implemented in PyTorch.

\paragraph{Existing assets and licenses.}
Table~\ref{tab:asset_licenses} lists the third-party datasets, codebases, and pretrained model
checkpoints used in this work, together with their licenses or terms of use. We do not redistribute
any third-party datasets or pretrained checkpoints.

\begin{table}[h]
\tiny
\centering
\caption{Existing assets used in this work and their licenses or terms of use.}
\label{tab:asset_licenses}
\begin{tabular}{lll}
\toprule
Asset & Use in this paper & License or terms of use \\
\midrule
ImageNet-1K &
LeJEPA PT, evaluation &
\makecell[l]{ImageNet Terms of Access:\\non-commercial research and educational use} \\

ADE20K &
Linear semantic-segmentation probing &
\makecell[l]{ADE20K Terms of Use:\\non-commercial research and educational use} \\

Pascal VOC 2012 &
Linear semantic-segmentation probing &
\makecell[l]{PASCAL VOC dataset terms / citation requirements;\\used for research benchmarking} \\

DINOv2 code and checkpoints &
Pretrained ViT sink analysis &
\makecell[l]{Apache License 2.0} \\

OpenCLIP code/checkpoints &
Pretrained ViT sink analysis &
\makecell[l]{MIT License for OpenCLIP code;\\checkpoint/model-card terms as released by OpenCLIP} \\

EVA code/checkpoints &
Pretrained ViT sink analysis &
\makecell[l]{MIT License for EVA repository;\\checkpoint/model-card terms as released by EVA} \\

LeJEPA code &
Model architecture and training framework &
\makecell[l]{CC BY-NC 4.0 License} \\

%Gated Attention code/method &
%Gated attention implementation reference &
%\makecell[l]{License of the authors' released repository, if used;\\otherwise paper citation only} \\
\bottomrule
\end{tabular}
\end{table}

\newpage

\section{Additional Experimental Results and Figures}
\label{app:additional_results}

\paragraph{Sink counts for the models of Table~\ref{tab:lejepa_full_results}.}
Table~\ref{tab:lejepa_sink_counts} gives the per-image sink counts quoted in Section~\ref{sec:mitigation}. We run the epoch-50 checkpoint of each model (Appendix~\ref{app:lejepa_table1}) on 1,984 ADE20K validation images (resize 256, center crop 224). For every image we count the (layer, head) pairs in layers 1--23 in which at least $90\%$ of the queries place at least 0.5 of their attention on the same key token; block 0 is excluded. A sink is NOP-like if the sink token's value-norm ratio $\|\bm{v}_s\|/\overline{\|\bm{v}\|}$ (its value norm over the head's mean value norm of the other tokens) is at most 0.2, and broadcast-like if it is at least 0.5. Sinks with intermediate ratios count only in the total. This rule relaxes Definition~1: it fixes $\epsilon=0.5$ but requires it on $90\%$ of the queries rather than all of them. Appendix~\ref{app:lejepa_ablations} recounts the same models under the literal definition.

\begin{table}[h]
\small
\centering
\caption{Sinks per image for the models of Table~\ref{tab:lejepa_full_results} (epoch-50 checkpoints; 1,984 ADE20K validation images; layers 1--23; counting rule in the text). ``All sinks'' includes sinks with intermediate value-norm ratios. ``On \texttt{[CLS]}'' counts sinks of any type located on the \texttt{[CLS]} token; these models have no registers.}
\label{tab:lejepa_sink_counts}
\begin{tabular}{lcccc}
\toprule
\textbf{Variant} & \textbf{NOP-like} & \textbf{Broadcast-like} & \textbf{All sinks} & \textbf{On \texttt{[CLS]}} \\
\midrule
Baseline                  & 0.25 & 0.09 & 0.45 & 0.34 \\
Gating                    & 0.00 & 3.80 & 3.88 & 2.02 \\
Global + Penalty          & 0.00 & 0.04 & 0.11 & 0.11 \\
Gating + Global + Penalty & 0.03 & 0.53 & 0.63 & 0.08 \\
\bottomrule
\end{tabular}
\end{table}

%%% FIGURE 7 %%%

\begin{figure*}[h!]
\centering
\includegraphics[width=0.9\linewidth]{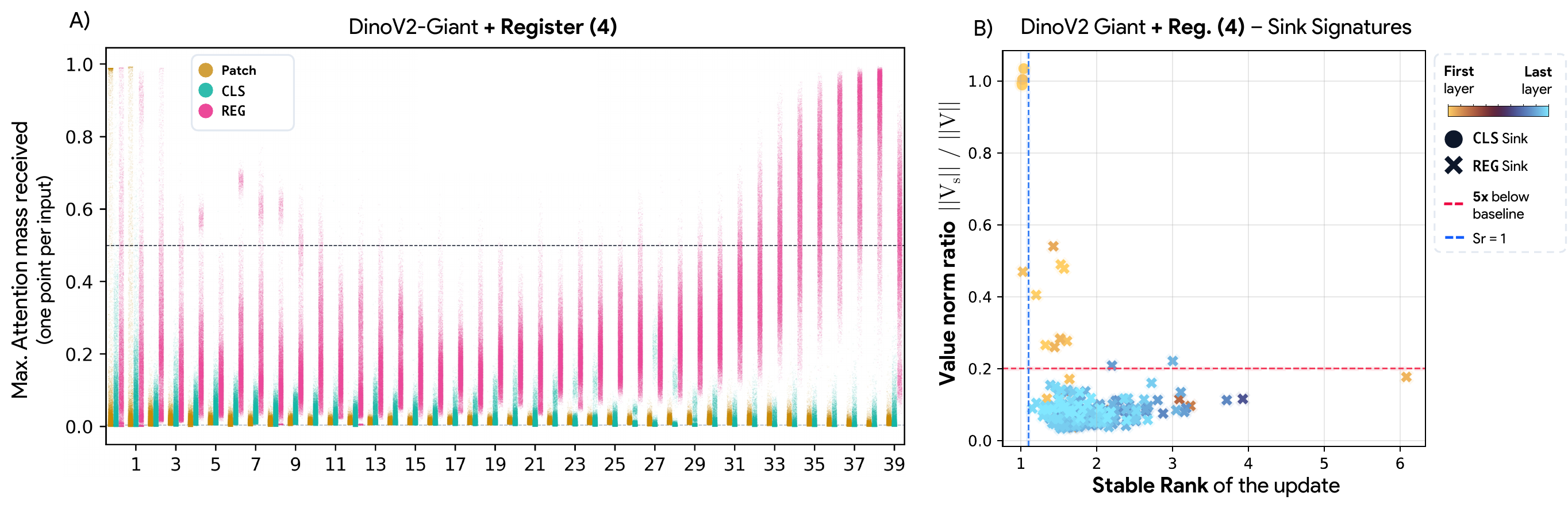}
\vspace{-2mm}
\caption{
\textbf{(Left) Registers absorb sink behavior.} In DINOv2-G + Reg.(4), register tokens (pink) capture nearly all attention mass across layers, displacing patch and \texttt{[CLS]} sinks.
\textbf{(Right) Registers inherit both regimes.} Register sinks cluster into the same two phenotypes: \nop~(low norm, majority) and broadcast (high norm, rank-1). Registers are repurposed for both mechanisms.
}
\label{fig:register_dino}
\end{figure*}

\begin{figure*}[h!]
\centering
\vspace{-4mm}
\includegraphics[width=\linewidth]{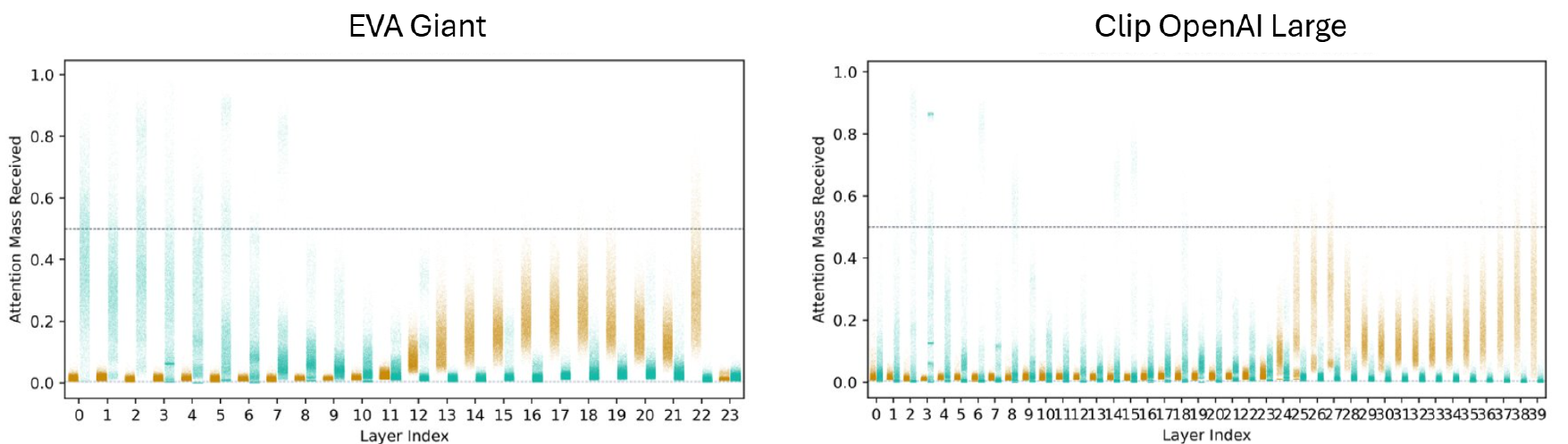}
\caption{\textbf{Sink token transition across layers.} EVA Giant and Clip OpenAI Large both exhibit a handoff pattern, although the pattern is more pronounced in EVA: \texttt{[CLS]} serves as the sink in early layers but yields to patch tokens in later layers.}
\label{fig:sink_distribution_app}
\end{figure*}

%\iffalse
\begin{figure}[h!]
\centering
\includegraphics[width=0.9\linewidth]{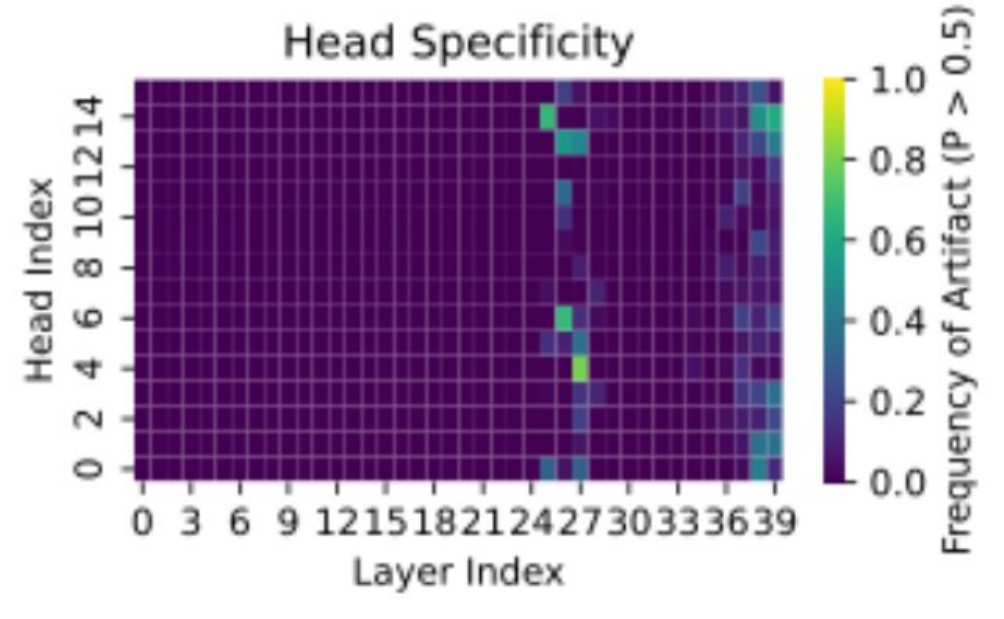}
\caption{\textbf{Head specialization in EVA Giant.} The entropy of values per layer shows that sink behavior is head-specific.}
\label{fig:head_specificity_eva}
\end{figure}

\begin{figure}[h!]
\centering
\includegraphics[width=0.9\linewidth]{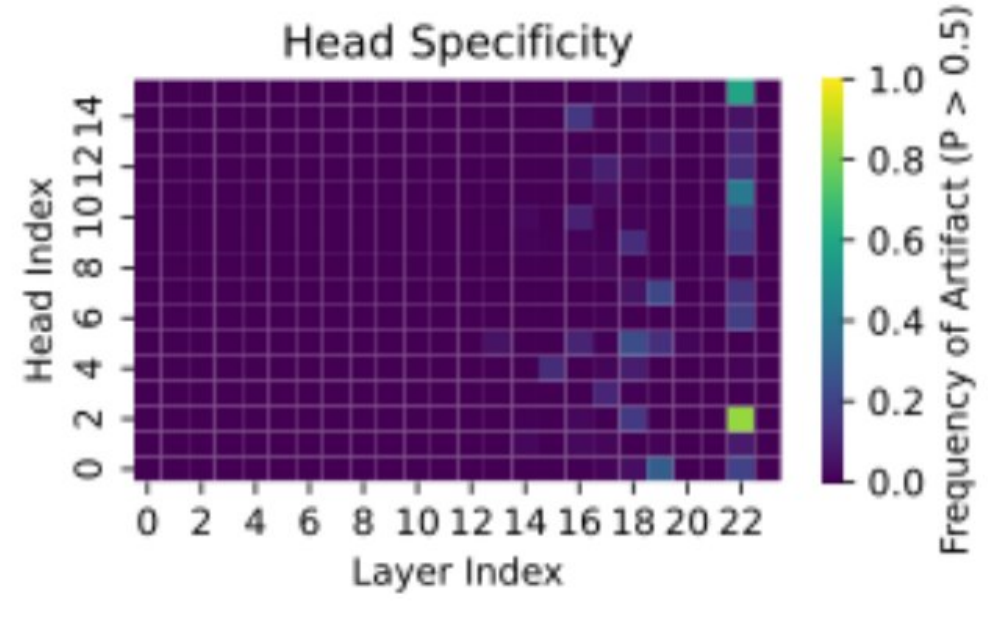}
\caption{\textbf{Head specialization in Clip OpenAI Large.} The entropy of values per layer shows that sink behavior is head-specific.}
\label{fig:head_specificity_clip}
\end{figure}

\begin{figure}[h!]
\centering
\includegraphics[width=0.95\linewidth]{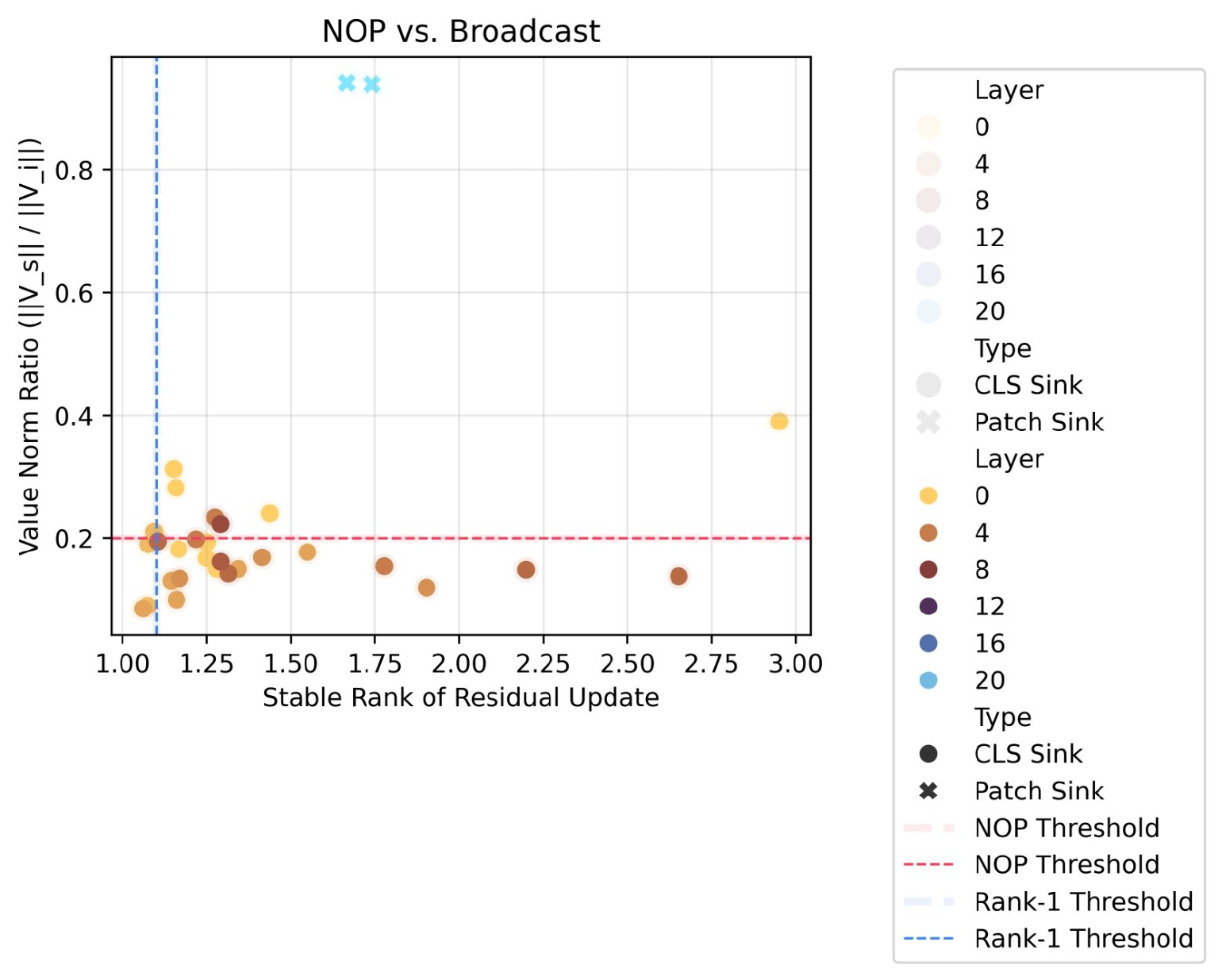}
\caption{\textbf{Dual Phenomenology in Clip OpenAI Large.} Sinks cluster into \nop~sinks (low norm, bottom) and broadcast sinks (moderate to high norm, $\approx$rank-1 update, left). Both regimes coexist within the same model.}
\label{fig:clip_phenom}
\end{figure}
%\fi

\begin{figure*}[h!]
    \centering
    \includegraphics[width=\textwidth]{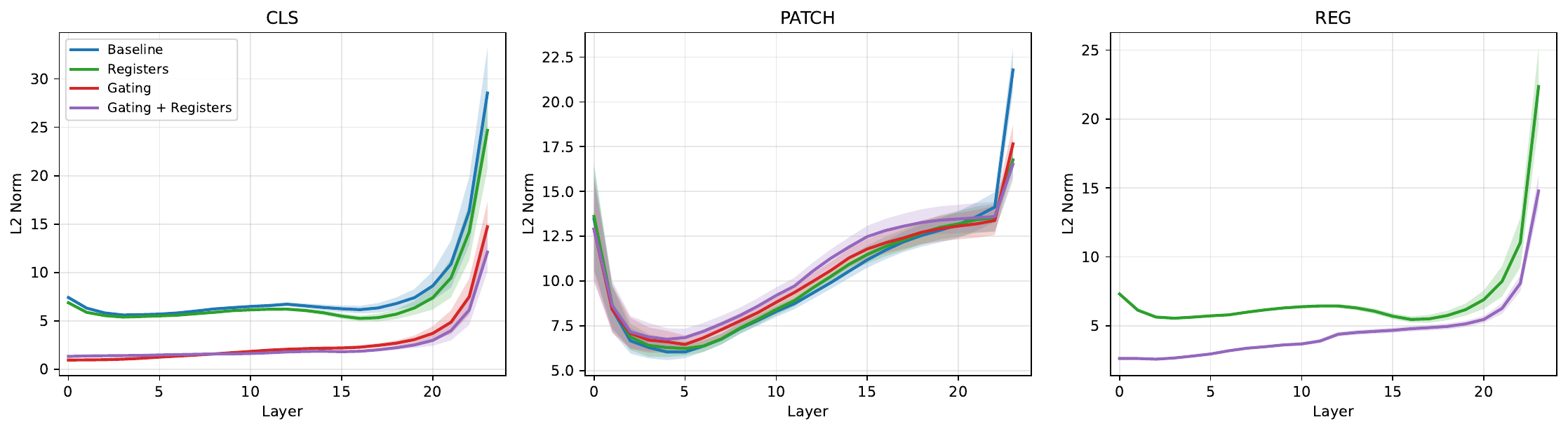}
    \caption{
    \textbf{Gating and registers mitigate high norms.}
    Distribution of token norms across layers on the ImageNet-1k validation set.
    }
    \label{fig:vitl_norm_comparison}
\end{figure*}

\iffalse
\begin{figure*}[h!]
\centering
\includegraphics[width=0.9\linewidth]{assets/def1_interventions.pdf}
\vspace{-2mm}
\caption{
\textbf{Causal interventions on sinks in pretrained ViTs.} We identify sinks in layers $\geq 1$ of DINOv2-L/G over 4095 ImageNet-1k validation images. We edit one head while keeping its attention matrix fixed, then measure the relative change of the final-layer patch tokens. Bars show the mean over heads, and dots show individual heads, colored by model. Blue bars change only the sink's value; gray bars are controls.
\textbf{(Left) \nop~sinks do nothing because their value is null.} Zeroing the sink value or replacing it with another image's changes the output by only $1.0\%$ and $0.6\%$. Placing an ordinary value at the sink ($8.0\%$) or removing attention to the sink ($3.5\%$) has a large effect (54 heads).
\textbf{(Right) Broadcast sinks are necessary but carry no per-image content.} Zeroing the sink value changes the output by $5.5\%$, comparable to removing attention to the sink ($5.3\%$). Swapping in another image's sink value changes it by only $1.4\%$, and this change does not point toward the donor (cosine $0.02$, vs.\ $0.94$ in the toy model), because the payload is nearly identical across images (33 heads). Zeroing a random non-sink value has no effect in either regime.
}
\label{fig:vit_causal_interventions}
\end{figure*}
\fi

%%%%%%%%%%%%%%%%%%%%%%%%%%%%%%%%%%%%%%%%%%%%%%%%%%%%%%%%%%%%

\section{Supplementary mitigation results and ablations}
\label{sec:ablations}

\subsection{Controlled-Task Causal Validation of the Global Pathway}
\label{app:global-pathway-causal}

We test whether the global-pathway intervention moves broadcast computation out of
softmax attention on the controlled task from Appendix~A.1. We train three matched
variants: an attention-only baseline, a model with the explicit global pathway, and
the same model with the all-layer common-mode penalty from Eq.~\ref{eq:broadcast_penalty}, using
$\lambda_B=1$. All variants use the same task, initialization, and training batches.
In this controlled experiment, the global pathway pools the normalized token
representations into a single state $g^\ell$ and broadcasts
$W_{O,G}^\ell g^\ell$ uniformly to every recipient; we omit the token-dependent
recipient gate to isolate the relocation of shared communication.

\paragraph{The penalty moves broadcast out of attention.}
The attention-only baseline reaches task MSE $0.3093$ and assigns approximately
$69\%$ of its attention to the selected source $j^\star$ in both layers. Adding the
global pathway alone leaves task error essentially unchanged ($0.3098$), but does
not eliminate attention-side broadcast: L2 still assigns $0.699$ attention to
$j^\star$. With the common-mode penalty, source-directed attention falls to $0.031$
in L1 and $0.008$ in L2, and L2 source--sink matching falls to zero. At the same
time, the L2 global pooling operation selects the ground-truth source in $77.3\%$
of examples and assigns it mean pooling weight $0.651$. The corresponding attention
update has RMS $1.6\times10^{-5}$, compared with $2.553$ for the global update.
Task MSE remains $0.3648$. Thus, providing a global pathway alone creates an
alternative broadcast route but does not force the model to abandon attention;
the common-mode penalty shifts the shared computation to the global branch.

\paragraph{Causal dependence shifts to the global pathway.}
We next intervene on the two communication branches separately. \emph{Attention
ablation} removes the projected attention update while preserving the global
update, whereas \emph{global ablation} sets the global update to zero while
preserving attention. We additionally replace the global state $g^\ell$ with that
of an independently sampled donor sequence and propagate the intervention through
the remainder of the network.

Figure~\ref{fig:global-pathway-causal} shows the resulting causal effects. In the
attention-only baseline, removing the L2 attention update strongly disrupts the
task. Adding the global pathway without the penalty distributes causal dependence
across both branches. After applying the penalty, however, removing the L2 attention
update leaves task MSE unchanged at $0.3648$, whereas removing the L2 global pathway
raises it to $1.6539$. The attention branch is therefore no longer necessary for
the learned broadcast computation.

Cross-example patching provides a content-specific test of this relocation.
Replacing the L2 global state with that of a donor sequence induces an output change
aligned with the desired donor-minus-recipient broadcast payload
(cosine similarity $0.803$). Correspondingly, donor-target MSE falls from $1.636$
for the unmodified recipient output to $0.369$ after global-state replacement.
By contrast, donor patching of the L2 attention sink after penalization has
essentially zero alignment with the desired payload. Together, these interventions
show that the penalty does not merely remove the visual attention sink: it relocates
the causal broadcast channel from softmax attention to the explicit global pathway.

\begin{figure}[h]
    \centering
    \includegraphics[width=\linewidth]{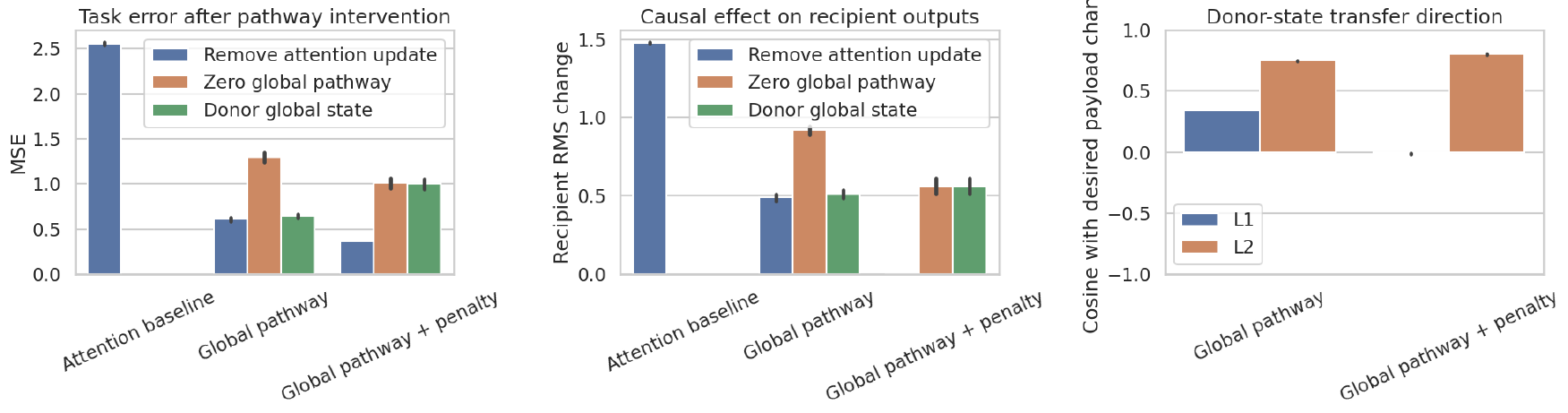}
    \caption{\textbf{Causal validation of broadcast relocation to the global
    pathway.}
    \textbf{Left:} task error after removing the attention update, ablating the
    global pathway, or replacing the global state with that of an independent donor.
    \textbf{Center:} RMS change in recipient outputs induced by the same
    interventions.
    \textbf{Right:} cosine similarity between the output change induced by
    donor-global-state replacement and the desired donor-minus-recipient broadcast
    payload, shown separately by layer.
    Adding the global pathway alone allows both attention and the auxiliary pathway
    to participate in broadcast. With the common-mode penalty, attention becomes
    causally dispensable while the L2 global pathway carries the transferable
    broadcast signal. Error bars denote standard errors over evaluation batches.}
    \label{fig:global-pathway-causal}
\end{figure}

\subsection{LeJEPA: Penalty Normalization, LayerScale, and Sensitivity to the Sink Rule}
\label{app:lejepa_ablations}

\paragraph{Penalty normalization.}
Eq.~\ref{eq:broadcast_penalty} normalizes the common-mode energy by the energy of the same attention update, with gradients through both. Normalizing by a stop-gradient energy did not train in LeJEPA. In one attempt we used $\lambda_B=1$ with the stop-gradient energy of the attention update as the denominator, fine-tuning baseline, gated and pathway checkpoints from epoch 69 (online probe 0.66). The probe fell to 0.002--0.024 within one epoch and reached at most 0.32 over the next 20--25 epochs. In a second attempt we used $\lambda_B=1$ with the stop-gradient energy of the whole block update as the denominator, training a pathway model from scratch. It stayed at chance (probe 0.0011 after one epoch, vs.\ 0.018 for the baseline). A stop-gradient denominator leaves only the gradient of the numerator, which the model can lower by shrinking the attention branch. The ratio in Eq.~\ref{eq:broadcast_penalty} is unchanged by such rescaling, and both penalized models of Table~\ref{tab:lejepa_full_results} train with it. Both stop-gradient runs also used $\lambda_B=1$ rather than $0.1$.

\paragraph{LayerScale initialization.}
Baseline and Gating use LayerScale initialized at $10^{-4}$ on the attention and MLP branches, as in the sweep; the two penalized models use $1.0$ (Table~\ref{tab:lejepa_table1_diff}). At $10^{-4}$ the penalty acts from the first step on an attention branch scaled down by $10^{4}$. In pilot runs it was then reduced largely through the learned LayerScale weights rather than through a change in attention. Baseline and gating-only models at LayerScale $1.0$ were not trained. Rows 1--2 and 3--4 of Table~\ref{tab:lejepa_full_results} therefore differ in initialization as well as in the intervention. Gating + Global + Penalty (at $1.0$) matches or exceeds Baseline (at $10^{-4}$) on all three probes, so the higher initialization does not by itself prevent baseline-level probe accuracy.

\paragraph{Sensitivity to the sink rule.}
Definition~1 requires every query to place at least $1-\epsilon$ of its attention on the sink token, whereas Table~\ref{tab:lejepa_sink_counts} uses a relaxed rule. Table~\ref{tab:lejepa_eps_sensitivity} recounts the same models and images under the literal definition for $\epsilon\in\{0.5, 0.3, 0.2, 0.1\}$, using the minimum over all queries of the attention to the dominant key token. For broadcast-like sinks the ordering Gating $>$ Gating + Global + Penalty $>$ Baseline $\geq$ Global + Penalty holds at every $\epsilon$. Under the literal definition the Baseline has at most 0.02 NOP-like sinks per image. The reduction of NOP-like sinks by gating reported in Section~\ref{sec:mitigation} is therefore visible only under the relaxed rule.

\begin{table}[h]
\small
\centering
\caption{Broadcast-like / NOP-like sinks per image for the models of Table~\ref{tab:lejepa_full_results} under the relaxed rule of Table~\ref{tab:lejepa_sink_counts} and under Definition~1 at several $\epsilon$ (every query places at least $1-\epsilon$ on the sink token). Same checkpoints, images, layers and value-norm cutoffs as Table~\ref{tab:lejepa_sink_counts}.}
\label{tab:lejepa_eps_sensitivity}
\setlength{\tabcolsep}{4pt}
\begin{tabular}{lccccc}
\toprule
 & \textbf{Relaxed} & \multicolumn{4}{c}{\textbf{Definition 1}} \\
\cmidrule(lr){3-6}
\textbf{Variant} & & $\epsilon=0.5$ & $\epsilon=0.3$ & $\epsilon=0.2$ & $\epsilon=0.1$ \\
\midrule
Baseline                  & 0.09 / 0.25 & 0.02 / 0.02 & 0.01 / 0.00 & 0.00 / 0.00 & 0.00 / 0.00 \\
Gating                    & 3.80 / 0.00 & 1.56 / 0.00 & 1.29 / 0.00 & 1.18 / 0.00 & 1.05 / 0.00 \\
Global + Penalty          & 0.04 / 0.00 & 0.00 / 0.00 & 0.00 / 0.00 & 0.00 / 0.00 & 0.00 / 0.00 \\
Gating + Global + Penalty & 0.53 / 0.03 & 0.31 / 0.00 & 0.17 / 0.00 & 0.11 / 0.00 & 0.07 / 0.00 \\
\bottomrule
\end{tabular}
\end{table}

\section{Extended Literature Review}

\subsection{Gated Attention}
\label{subsec:gated_attention}

We briefly describe the \emph{gated attention} modification of softmax attention \cite{qiu2025gated} used throughout this paper. Gated attention introduces an additional \emph{gate} computed from a (typically pre-normalized) hidden-state signal and uses it to modulate one of the intermediate quantities in the attention block.

\paragraph{Default gated-softmax output.} Unless stated otherwise, we use \emph{head-specific, multiplicative sigmoid gating} applied to the SDPA (softmax-attention) output. Starting from the standard single-head formulation in Eq.~\ref{eq:attention}, let $\mX \in \mathbb{R}^{n \times d}$ denote the matrix whose $i$-th row is $\vx_i^\top$, and let $\mO = \A \mV$ denote the ungated head output. We form the gated head output
\begin{equation}
\label{eq:gated_output}
\widetilde{\mO} \;=\; \mO \odot \mG \;=\; (\A \mV) \odot \sigma(\mX \mW_\theta),
\end{equation}
where $\mG = \sigma(\mX \mW_\theta) \in \mathbb{R}^{n \times d_h}$, $\mW_\theta$ are learnable parameters, $\sigma(\cdot)$ is the sigmoid function, and $\odot$ denotes elementwise multiplication. This replaces the ungated head output in the attention block. Further implementation details are provided in Appendix~D.

In our framing, output gating provides a direct mechanism for conditional suppression, allowing a head to implement an approximate NOP without routing attention to a null token. It does not, however, provide a separate route for shared communication, which motivates our complementary global pathway for broadcast mitigation.

The interpretation of concentrated attention on a special token as a null operation predates the attention-sink literature: \citet{clark2019what} observed that a large fraction of BERT heads attend almost exclusively to \texttt{[SEP]} and hypothesized that this functions as a ``no-op'' invoked when the head's operation does not apply to the current token. Our adaptive-NOP hypothesis makes this reading identifiable, by pairing it with an explicit value-norm condition and testing it causally rather than inferring it from the attention map alone. Gating is one of several architectural responses to the same observation; alternatives include supplying an explicit null destination \cite{xiao2024efficient} or relaxing the row-stochasticity constraint outright (Appendix~\ref{subsec:non-softmax_attention}).

\subsection{Register Tokens and Global Communication}
\label{subsec:register_tokens}

We also consider the \emph{register token} modification of a ViT \cite{darcet2023vision}, which appends a small number of additional learnable tokens to the input sequence. Concretely, let $\vx_1,\dots,\vx_n \in \mathbb{R}^d$ denote the patch-token representations (and optionally a class token, treated as an additional special token in the same way). We introduce $r$ register tokens with learnable embeddings $\vr_1,\dots,\vr_r \in \mathbb{R}^d$ that are independent of the input image and are appended to the sequence after patch embedding. Attention is then computed without any other change on the extended sequence, and the same residual updates are applied to both the original and register tokens. At the end of the network, the register tokens are \emph{discarded}: downstream representations use only the original tokens (e.g., patch tokens and optionally \texttt{[CLS]}), while registers serve purely as internal workspace during the forward pass. Further implementation details are provided in Appendix~D.

Registers therefore provide dedicated workspace for global computation, but they remain sequence positions and can themselves attract concentrated attention. Consistent with this distinction, our analysis finds that register-equipped ViTs preserve both sink regimes: register sinks can exhibit either suppressed value norms consistent with adaptive NOP or content-scale values and approximately low-rank updates consistent with broadcast. Thus, registers can relocate sink computation without eliminating the single-position routing pattern.

\paragraph{Appended tokens as workspace.} Registers belong to a longer line of work that augments the sequence with input-independent slots. \citet{sukhbaatar2019augmenting} add persistent memory vectors inside the attention operation, and \citet{burtsev2020memory} prepend dedicated \texttt{[mem]} tokens to the input; in the autoregressive setting, \citet{xiao2024efficient} show that reserving a single learnable sink token both stabilizes streaming inference and absorbs the attention mass that would otherwise fall on the first position. These constructions differ in where the slot lives and how it is trained, but they share the property we highlight here: because the slot occupies a sequence position, it can be selected by the same routing mechanism as any other token, and therefore inherits rather than eliminates the single-position bottleneck. Our diagnostics show that this is what happens in practice, with registers being used for both adaptive NOP and broadcast.

More broadly, explicit global-context mechanisms have long been used to aggregate information across spatial positions and return it to local representations, including non-local, global-context and squeeze-and-excitation modules \cite{wang2018non,cao2019gcnet,hu2018squeeze}. The permutation-invariant pooling of our global pathway is likewise related to attention-based set pooling \cite{lee2019set} and to architectures that route global information through a latent array held outside the token sequence \cite{jaegle2021perceiver}; class-attention layers \cite{touvron2021going} pursue a related separation by decoupling the class-token readout from patch-to-patch mixing. Our global pathway is related in spirit, but is introduced specifically as a position-free alternative to broadcast through attention. Crucially, providing such a pathway alone does not determine where the broadcast computation is carried; we therefore pair it with the common-mode penalty in Eq.~\ref{eq:broadcast_penalty}, which discourages attention from carrying the shared update while leaving attention concentration itself unpenalized.

\subsection{Non-softmax Attention}
\label{subsec:non-softmax_attention}

A complementary line of work modifies or replaces the softmax normalization itself. These approaches alter how attention weights are formed, rather than introducing an additional communication pathway.

\paragraph{1. Clipped-softmax attention.}

Clipped-softmax \cite{bondarenko2023quantizable} replaces the usual row-wise softmax with a stretched-and-clipped variant. Given standard softmax weights $\tilde{\A}=\softmax(\mS)$ and hyperparameters $\zeta\ge 1$, $\gamma\le 0$, it defines
\begin{equation}
\A_{\text{clip}}
=
\operatorname{clip}\big((\zeta-\gamma)\,\tilde{\A}+\gamma,\;0,1\big).
\end{equation}
Softmax outputs are stretched and clipped back to $[0,1]$, allowing exact zeros and ones with finite score ranges. This gives attention heads a direct way to assign zero mass to all positions, alleviating the need for a dedicated null destination \cite{bondarenko2023quantizable}.

\paragraph{2. ReLU attention.}

ReLU attention \cite{wortsman2023replacing} replaces softmax with a pointwise ReLU and sequence-length scaling:
\begin{equation}
\A_{\text{ReLU}}
=
\frac{1}{n}\operatorname{ReLU}(\mS).
\end{equation}
Rows are no longer normalized to sum to one, and entries can be exactly zero. The factor $1/n$ controls the scale of the attention output as sequence length changes. More generally, point-wise attention families replace softmax by a generic elementwise nonlinearity $h$ together with length-dependent scaling, $\A_{h,\alpha}=n^{-\alpha}h(\mS)$ \cite{saratchandran2024rethinking}.

\paragraph{3. Polynomial attention.}

Polynomial attention \cite{saratchandran2024rethinking} similarly replaces softmax with an elementwise polynomial of the scores. A representative form is
\begin{equation}
\A_{\text{poly}}
=
\sqrt{\frac{1}{n}}\,\mS^{\odot p},
\end{equation}
where $\mS^{\odot p}$ denotes the elementwise power. The sequence-length scaling is chosen to control the magnitude of the attention matrix and its Jacobian while dropping the non-negativity and row-normalization constraints imposed by softmax.

\paragraph{4. Sigmoid self-attention.}

Sigmoid attention \cite{ramapuram2024theory} replaces row-wise softmax by an elementwise sigmoid,
\begin{equation}
\A_{\text{sig}}
=
\sigma(\mS+b),
\end{equation}
where $b$ is a learnable or fixed bias that may depend on sequence length. Rows are not normalized and entries lie strictly in $(0,1)$; suitable choices of $b$ and normalization can keep attention magnitudes controlled. This use of a sigmoid to form \emph{attention weights} is distinct from the sigmoid output gating in Appendix~\ref{subsec:gated_attention}, which leaves the softmax attention weights unchanged and instead modulates the resulting attention output.

\paragraph{5. Constraining the logits rather than the normalizer.}

A related family keeps softmax but controls the scale of the pre-softmax scores. \citet{zhai2023stabilizing} show that low attention entropy accompanies training instability and propose $\sigma$Reparam, which bounds the spectral norm of the query--key parameters. This is a direct constraint on the spectral norm $\|\Theta\|_2$ that governs the sink logit magnitude in our analysis, and it makes the trade-off in our framing explicit: the spectral spike that implements a reliable NOP gate is the same quantity whose growth these methods suppress for stability. Methods that instead target the activation side of the same product, by minimizing outlier features during training \cite{he2024understanding}, constrain the complementary ``massive activation'' regime.

These methods therefore modify the weighting or normalization rule within attention itself. Our global-pathway intervention instead retains standard softmax self-attention and introduces a separate route for shared communication. Moreover, the global pathway is not itself ``softmax-free'': its permutation-invariant pooling operation uses softmax weights. The distinction is architectural rather than a replacement of the attention normalization.

\subsection{Other Related Works}

\paragraph{Geometric and optimization perspectives.}
Beyond the sink literature summarized in Section~2.2, related work connects transformer dynamics to spectral concentration, outlier activations, dormant heads, rank collapse, and oversmoothing \cite{queipo2025attention,cancedda2024spectral,sun2024massive,sandoval2025using,jacobs2025block,saada2024mind,wang2022anti}. Rank collapse in particular has a well-developed theory: repeated self-attention drives token representations toward a rank-one subspace at a doubly exponential rate absent skip connections and MLPs \cite{dong2021attention}, an effect that can be characterized through signal propagation at initialization \cite{noci2022signal}, and whose continuous-time limit provably clusters tokens onto a small number of ``leader'' positions \cite{geshkovski2023emergence}. These phenomena are closely related to, but distinct from, the mechanisms studied here. In particular, a low-rank \emph{attention update} caused by broadcast does not by itself imply collapse of the token representations: as shown in Lemma~\ref{cor:similarity}, a perfectly shared additive update preserves centered token variance. Likewise, large activations or a dominant query--key direction can support adaptive routing to a NOP sink without themselves defining the NOP computation. We also note that the head dimension imposes its own ceiling on the rank a single head can write \cite{bhojanapalli2020low}; the stable rank $\approx 1$ we measure in Section~3.2 is well below this ceiling and is therefore a property of the learned solution rather than of the architecture.

\paragraph{Attention weights as evidence about computation.}
Our starting point, that an attention map underdetermines the underlying algorithm, is shared with a literature on whether attention constitutes explanation \cite{jain2019attention,wiegreffe2019attention,akula2022attention,bibal2022attention}. Most relevant to our diagnostics, \citet{kobayashi2020attention} argue that attention weights should not be read in isolation and that the informative quantity is the norm-weighted contribution $\|\A_{ij}\vv_j\|$, which reveals that heads attending heavily to special tokens often transmit almost nothing. Our value-norm diagnostic is the same measurement placed in a mechanistic frame: we show that a near-null value is not merely correlated with a null update but is the unique way to realize one under a perfect sink, and our interventions test that claim causally.

\paragraph{Head specialization and circuits.}
The sparse, head-specific structure we report in Section~4 is consistent with evidence that attention heads specialize and that many can be pruned with little loss \cite{voita2019analyzing,michel2019sixteen}, and with mechanistic accounts of individual heads implementing identifiable algorithms \cite{elhage2021mathematical}. Dormant-head analyses in LLMs \cite{sandoval2025using} describe a closely related regime; our contribution is to separate, within the class of sink heads, those that are conditionally inactive from those that are actively broadcasting.

\paragraph{Causal interventions.}
Our intervention protocol follows standard practice in mechanistic interpretability, where activation patching and its variants are used to localize the components responsible for a behavior \cite{meng2022locating,wang2023interpretability}. We follow the recommendations of \citet{zhang2024towards} in separating interventions on the payload carried at a destination from interventions on the routing itself, and in reporting a donor-patched target alongside raw output change, so that a large perturbation is not mistaken for transfer of the hypothesized content.

\paragraph{Mitigation strategies beyond register tokens and gated attention.}
Prior work has also explored modifying the attention mechanism to reduce pathological concentration or make null updates easier to express, including clipping or rescaling attention logits, reserving an explicit sink position, and replacing softmax with alternatives that permit exact or near-exact zero weights \cite{bondarenko2023quantizable,wortsman2023replacing,ramapuram2024theory,xiao2024efficient}. These approaches primarily alter how attention mass is allocated. Our broadcast intervention instead leaves standard softmax attention intact and discourages only the \emph{shared component} carried by the attention update, while providing a separate global pathway to preserve that computation.

\clearpage

% \section*{AI use statement}

% In this work, we used generative AI tools for designing or providing feedback on research methodology or experiments, implementing methods,  generating synthetic data sets, assisting in the writing of proofs, and for interpreting results. 
%
% We have not used generative AI tools for cleaning and reformating datasets or proposing and refining hypotheses, help develop theoretical models or conceptual frameworks, formulate mathematical claims, assist with translation, provide critical ingredients for proving mathematical claims, or support qualitative and thematic data analysis.
%
% Additionally, we used generative AI tools for creating or modifying scientific figures or images, suggesting experimental parameters, creating or editing software code, and editing a research paper to improve readability. We have reviewed all AI-assisted work. LLM-generated code and LLM-prepared datasets were verified and tested for correctness by 2 authors. We take responsibility for the final content of this work, including text, claims or artifacts produced with the aid of generative AI.

%\input{checklist.tex}

\end{document}